\documentclass[11pt,letterpaper]{article}

\usepackage[letterpaper,margin=1in]{geometry}
\usepackage[utf8]{inputenc}
\usepackage[T1]{fontenc}
\usepackage{lmodern}
\usepackage{microtype}
\usepackage{acs-preprint}
\preprintshorttitle{Divergence Timing and Cumulative Disagreement under KV-Cache Eviction}
\preprintauthors{Luo and Yu}
\usepackage{amsmath,amssymb,amsthm,mathtools}
\usepackage{booktabs,array}
\usepackage{graphicx,subcaption}
\usepackage{algorithm,algpseudocode}
\usepackage{enumitem,xcolor,float,placeins}
\usepackage{tikz}
\usepackage[numbers,sort&compress]{natbib}
\usepackage{xurl}
\usepackage{hyperref}
\graphicspath{{figures/}}
\definecolor{figred}{HTML}{BC4A36}
\definecolor{figblue}{HTML}{236A9D}
\definecolor{figpurple}{HTML}{775496}
\definecolor{figexposure}{HTML}{3E8191}
\definecolor{figrate}{HTML}{D39A56}
\definecolor{figfloor}{HTML}{596270}
\newcommand{\figurefont}{\fontsize{10}{11.5}\selectfont}
\newcommand{\figline}[1]{\tikz[baseline=-.5ex]\draw[#1,line width=.9pt](0,0)--(.35,0);}
\newcommand{\figblock}[1]{\tikz[baseline=-.5ex]\fill[#1](0,-.04)rectangle(.30,.12);}
\newcommand{\actionlegend}{\figurefont\figline{figred}\ SnapKV-512\quad
 \figline{figblue}\ SnapKV-50\%\quad\figline{figpurple}\ Recent-50\%}
\makeatletter
\renewcommand{\theHALG@line}{\thealgorithm.\arabic{ALG@line}}
\makeatother
\hypersetup{
  colorlinks=true,
  linkcolor=blue!50!black,
  citecolor=green!35!black,
  urlcolor=blue!55!black,
  pdftitle={Divergence Timing and Cumulative Disagreement under KV-Cache Eviction},
  pdfauthor={Xinyue Luo, Fei Yu}
}

\newtheorem{theorem}{Theorem}[section]
\newtheorem{proposition}[theorem]{Proposition}

\newtheorem{corollary}[theorem]{Corollary}
\theoremstyle{definition}
\newtheorem{definition}[theorem]{Definition}
\theoremstyle{remark}

\newcommand{\E}{\mathbb{E}}
\newcommand{\Pcal}{\mathbb{P}}
\newcommand{\TV}{\operatorname{TV}}
\newcommand{\one}{\mathbf{1}}
\newcommand{\Rsmc}{R_{\mathrm{SMC}}}

\title{Divergence Timing and Cumulative Disagreement\\
under KV-Cache Eviction}
\author{Xinyue Luo\quad Fei Yu\thanks{Corresponding author.}\\[3pt]
  \normalsize Ant Group\\[2pt]
  \small\href{mailto:mangduo.lxy@antgroup.com}{mangduo.lxy@antgroup.com}\quad
  \href{mailto:fred.yf@antgroup.com}{fred.yf@antgroup.com}}
\date{}

\begin{document}
\maketitle
\begin{abstract}
KV-cache eviction perturbs the conditional token distributions governing
autoregressive generation. We investigate how first-divergence timing and
subsequent token mismatch determine cumulative disagreement.
We derive an exact decomposition under a specified stepwise maximal coupling:
the expected mismatch fraction equals a first-mismatch contribution plus post-divergence
exposure multiplied by its mismatch rate. An explicit construction over
unrestricted autoregressive kernel pairs realizes the sharp interval of
risks compatible with a finite divergence-aligned observation window.
Residual-branch conditional Monte Carlo provides unbiased joint estimates
of occurrence, occupation, and window/tail contributions, with per-replicate
variance dominance for total token loss. Complete trajectories from
Meta-Llama-3.1-8B-Instruct and Qwen2.5-7B-Instruct show that SnapKV at
50\% retention enters divergence later and less often than SnapKV-512
or recent-token retention with the same 50\% prompt-cache budget, while
post-divergence total variation (TV) remains high. In an exploratory
analysis of 288 documents, post-divergence exposure accounts for 85--90\%
of four aggregate mismatch gaps. On 288 independent documents at 90\%
retention, prespecified comparisons show higher branch-aligned TV in the
late than in the early window in both models.
\end{abstract}

\section{Introduction}
\label{sec:introduction}

KV-cache eviction changes a language model's continuation by changing the
information retained during decoding. Attention, retrieval structure, and
forward-looking signals guide which entries to keep
\citep{li2024snapkv,fu2025headkv,tang2025razorattention,kai2026infokv,ahn2026lookaheadkv}.
The resulting perturbation unfolds over time: two continuations can agree
initially, diverge at a later token, and then spend different amounts of
time in disagreement. Our question is how \emph{entry into divergence}
and \emph{disagreement after entry} jointly determine the cumulative
consequences of eviction.

A simple example makes the distinction exact. Let a reference always
emit zero. Two interventions emit a first-position one with probability
\(\varepsilon\). Conditional on that departure, they continue with zeros
or ones, respectively; otherwise both continue with zeros. Both have the same
first-divergence law and sequence total variation (TV), \(\varepsilon\),
but their expected mismatch fractions are \(\varepsilon/H\) and
\(\varepsilon\). The first departure and the occupation of disagreement
describe different aspects of generation (Figure~\ref{fig:theory-story}).
This distinction gives a trajectory-level perspective on error accumulation
and self-recovery in autoregressive models
\citep{he2021exposure,arora2022exposure}.

\begin{figure}[t]
\centering
\begin{subfigure}[b]{.485\linewidth}
\centering\begin{tikzpicture}[x=1cm,y=1cm,font=\figurefont,line cap=round]
\path[use as bounding box] (0,0) rectangle (6.7,2.05);
\node[inner sep=0pt,] at (3.35000,1.92000) {$H=6,\quad \Pr(\tau=1)=\varepsilon$};
\node[inner sep=0pt,anchor=east] at (1.65000,1.40000) {Short};
\filldraw[fill=figfloor,draw=figfloor,line width=.5pt] (2.05000,1.40000) circle[radius=.115cm];
\filldraw[fill=white,draw=black!35,line width=.5pt] (2.60000,1.40000) circle[radius=.115cm];
\filldraw[fill=white,draw=black!35,line width=.5pt] (3.15000,1.40000) circle[radius=.115cm];
\filldraw[fill=white,draw=black!35,line width=.5pt] (3.70000,1.40000) circle[radius=.115cm];
\filldraw[fill=white,draw=black!35,line width=.5pt] (4.25000,1.40000) circle[radius=.115cm];
\filldraw[fill=white,draw=black!35,line width=.5pt] (4.80000,1.40000) circle[radius=.115cm];
\node[inner sep=0pt,] at (5.95000,1.40000) {$\varepsilon/6$};
\node[inner sep=0pt,anchor=east] at (1.65000,0.85000) {Recurring};
\filldraw[fill=figfloor,draw=figfloor,line width=.5pt] (2.05000,0.85000) circle[radius=.115cm];
\filldraw[fill=white,draw=black!35,line width=.5pt] (2.60000,0.85000) circle[radius=.115cm];
\filldraw[fill=figexposure,draw=figexposure,line width=.5pt] (3.15000,0.85000) circle[radius=.115cm];
\filldraw[fill=white,draw=black!35,line width=.5pt] (3.70000,0.85000) circle[radius=.115cm];
\filldraw[fill=figexposure,draw=figexposure,line width=.5pt] (4.25000,0.85000) circle[radius=.115cm];
\filldraw[fill=white,draw=black!35,line width=.5pt] (4.80000,0.85000) circle[radius=.115cm];
\node[inner sep=0pt,] at (5.95000,0.85000) {$\varepsilon/2$};
\node[inner sep=0pt,anchor=east] at (1.65000,0.30000) {Persistent};
\filldraw[fill=figfloor,draw=figfloor,line width=.5pt] (2.05000,0.30000) circle[radius=.115cm];
\filldraw[fill=figexposure,draw=figexposure,line width=.5pt] (2.60000,0.30000) circle[radius=.115cm];
\filldraw[fill=figexposure,draw=figexposure,line width=.5pt] (3.15000,0.30000) circle[radius=.115cm];
\filldraw[fill=figexposure,draw=figexposure,line width=.5pt] (3.70000,0.30000) circle[radius=.115cm];
\filldraw[fill=figexposure,draw=figexposure,line width=.5pt] (4.25000,0.30000) circle[radius=.115cm];
\filldraw[fill=figexposure,draw=figexposure,line width=.5pt] (4.80000,0.30000) circle[radius=.115cm];
\node[inner sep=0pt,] at (5.95000,0.30000) {$\varepsilon$};
\end{tikzpicture}
\caption{Same entry, different occupation}
\end{subfigure}\hfill
\begin{subfigure}[b]{.485\linewidth}
\centering\begin{tikzpicture}[x=1cm,y=1cm,font=\figurefont,line cap=round]
\path[use as bounding box] (0,0) rectangle (6.7,2.05);
\node[inner sep=0pt,] at (3.35000,1.92000) {Illustrative path: $H=8,\ \tau=3$};
\path[fill=figexposure!12] (2.85000,0.75000) rectangle (6.45000,1.35000);
\filldraw[fill=white,draw=black!35,line width=.5pt] (0.55000,1.05000) circle[radius=.115cm];
\filldraw[fill=white,draw=black!35,line width=.5pt] (1.35000,1.05000) circle[radius=.115cm];
\filldraw[fill=figfloor,draw=figfloor,line width=.5pt] (2.15000,1.05000) circle[radius=.115cm];
\filldraw[fill=white,draw=black!35,line width=.5pt] (2.95000,1.05000) circle[radius=.115cm];
\filldraw[fill=figexposure,draw=figexposure,line width=.5pt] (3.75000,1.05000) circle[radius=.115cm];
\filldraw[fill=white,draw=black!35,line width=.5pt] (4.55000,1.05000) circle[radius=.115cm];
\filldraw[fill=figexposure,draw=figexposure,line width=.5pt] (5.35000,1.05000) circle[radius=.115cm];
\filldraw[fill=figexposure,draw=figexposure,line width=.5pt] (6.15000,1.05000) circle[radius=.115cm];
\node[inner sep=0pt,] at (0.55000,0.58000) {1};
\node[inner sep=0pt,] at (1.35000,0.58000) {2};
\node[inner sep=0pt,] at (2.15000,0.58000) {3};
\node[inner sep=0pt,] at (2.95000,0.58000) {4};
\node[inner sep=0pt,] at (3.75000,0.58000) {5};
\node[inner sep=0pt,] at (4.55000,0.58000) {6};
\node[inner sep=0pt,] at (5.35000,0.58000) {7};
\node[inner sep=0pt,] at (6.15000,0.58000) {8};
\node[inner sep=0pt,] at (3.35000,0.10000) {5 exposed steps; 3 later mismatches};
\end{tikzpicture}
\caption{Entry sets the time available}
\end{subfigure}
\par\smallskip
\(\Rsmc=O_H+E_H C_H\)\quad
(first mismatch + exposure \(\times\) mismatch rate)
\caption{First divergence and subsequent occupation jointly determine cumulative
disagreement. Filled/open circles denote token mismatches/matches.
(a) Constructive paths conditional on departure; otherwise all tokens match.
Identical divergence laws give different expected losses. (b) An illustrative
path has five exposed steps, including two local matches. The identity averages
over complete trajectories (Equation~\ref{eq:exposure-accounting}).}
\label{fig:theory-story}
\end{figure}
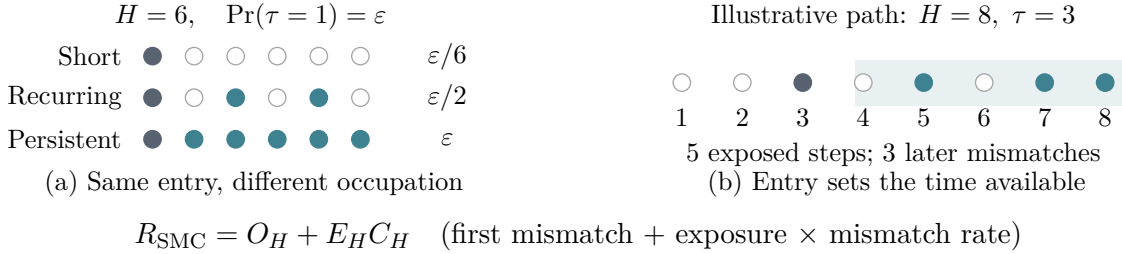

We make the trajectory comparison precise through recursive
overlap--residual maximal coupling
\citep{goldstein1979maximal,levin2017markov}.
Reference kernels \(P\) and intervention kernels \(Q^a\) generate paired
continuations, each with its original marginal law. Their cumulative
positional disagreement is
\begin{equation}
 \Rsmc(a)=\E\!\left[\frac1H\sum_{t=1}^H
 \one\{X_t\ne Y_t^a\}\right].
 \label{eq:intro-risk}
\end{equation}
First divergence separates mismatch at entry from subsequent occupation,
which factors into post-divergence exposure and its mismatch rate.
For an exact profile of the divergence law and the first \(L\) mismatch
marginals of each branch, we construct, over unrestricted autoregressive
kernel pairs, every compatible risk in \([R_L,R_L+B_L]\).
Here \(R_L\) is the observed contribution and \(B_L\) the remaining
probability-weighted trajectory capacity. Residual-branch conditional
Monte Carlo jointly estimates occurrence, occupation, and tail contributions
under the same coupling.

Complete trajectories from two language models show how these components
interact under persistent eviction. In an exploratory study, SnapKV at
50\% retention enters divergence later and less often than SnapKV-512
or recent-token retention with the same 50\% prompt-cache budget; exposure
contributes 85--90\% of four aggregate mismatch gaps. Within each action's
fixed cohort with 64 observable post-divergence steps, late-window TV
exceeds early-window TV. An independent study at 90\% retention finds the
same direction in prespecified branch comparisons. These temporal patterns
explain how smaller cumulative differences coexist with high discrepancy
after divergence.

\section{Related Work}
\label{sec:related-work}

KV-cache compression uses attention sinks, heavy hitters, retrieval heads,
observation windows, and heterogeneous budgets to select retained state
\citep{xiao2024streamingllm,zhang2023h2o,li2024snapkv,fu2025headkv,
feng2025adakv,qin2025cake}.
Quantization and low-rank projection offer complementary compression
axes~\citep{liu2024kivi,chang2025palu}; serving evaluations study their
quality--memory--latency trade-offs~\citep{wei2025rethinking}.
InfoKV evaluates the influence of KV removal on future predictive
distributions along a fixed token sequence, while LookaheadKV uses
learned predictors of future attention
\citep{kai2026infokv,ahn2026lookaheadkv}.
Worst-case space bounds characterize compression limits under general
attention inputs~\citep{haris2025barriers}.

Feedback, error accumulation, and self-recovery have been studied in
sequential prediction and language generation
\citep{ross2011dagger,bengio2015scheduled,he2021exposure,arora2022exposure}.
Sequence TV estimation~\citep{price2026tv} measures distance between marginal
sequence laws; maximal agreement couplings
\citep{vollering2016maximal,ernst2019mexit} maximize the probability of
agreement through each time. One-step maximal
coupling realizes TV as mismatch probability~\citep{goldstein1979maximal};
token couplings also support speculative decoding
\citep{leviathan2023speculative,sun2023spectr} and control generation randomness
in LLM evaluation~\citep{benz2026coupled}.

Partial identification characterizes the target values compatible with
available observations
\citep{manski2003partial,stoye2007minimax,khan2024offpolicy}.
Conditional Monte Carlo reduces simulation variance by integrating out
components of the sampling randomness
\citep{blackwell1947conditional,casella1996raoblackwellisation,glasserman2001survival}.
The present study examines how first-divergence timing and subsequent
mismatch jointly determine cumulative discrepancy under persistent
KV-cache eviction.

\section{From Divergence Timing to Cumulative Disagreement}
\label{sec:information}

First-divergence timing and subsequent mismatch determine different parts
of cumulative discrepancy. We separate these contributions under a specified
coupling, then characterize the risks compatible with a finite observation
window after divergence.

\subsection{Divergence timing, exposure, and occupation}
\label{sec:timing-occupation}

Fix a prompt, an intervention \(a\), a finite token alphabet \(\mathcal V\),
and an integer horizon \(H\ge1\). Empty sums are zero, and conditional
quantities on null events are zero. For histories
\(x,y\in\mathcal V^{t-1}\), write
\[
 p(u)=P_t(u\mid x),\quad q(u)=Q_t^a(u\mid y),\quad
 c(u)=\min\{p(u),q(u)\},\quad
 e=1-\sum_u c(u)=\TV(p,q).
\]
The overlap--residual coupling has joint transition kernel
\begin{equation}
 \Gamma_t(u,v\mid x,y)=c(u)\one\{u=v\}
 +\begin{cases}
   \dfrac{(p(u)-c(u))(q(v)-c(v))}{e},&e>0,\\[3pt]
   0,&e=0.
  \end{cases}
 \label{eq:coupling-kernel}
\end{equation}
Its marginals are \(p,q\), and the residual supports are disjoint; hence
\(\Gamma_t\{u\ne v\mid x,y\}=e\)
\citep[Proposition~4.7]{levin2017markov}. Recursive application defines the
trajectory law \(\Gamma(P,Q^a)\). All probabilities below use this law;
we suppress \(a\) for single-action statements.

Let \(I_t=\one\{X_t\ne Y_t\}\),
\(D_H=H^{-1}\sum_{t=1}^H I_t\), and
\(\mathcal F_{t-1}=\sigma(X_{<t},Y_{<t})\). With
\(\delta_t(x,y)=\TV(P_t(\cdot\mid x),Q_t^a(\cdot\mid y))\),
Equation~\eqref{eq:coupling-kernel} and iterated expectation give
\begin{equation}
 \begin{aligned}
  \E[I_t\mid\mathcal F_{t-1}]&=\delta_t(X_{<t},Y_{<t}),\\
  \Rsmc:=\E[D_H]&=\frac1H\sum_{t=1}^H
                    \E[\delta_t(X_{<t},Y_{<t})].
 \end{aligned}
 \label{eq:cross-history-tv}
\end{equation}
This risk integrates local TV along the actual paired histories. Before
divergence, \(\delta_t(h,h)\) compares distributions on a common prefix;
after divergence, both generated histories enter the comparison.

Define
\[
 \tau=\inf\{t\le H:I_t=1\},\qquad
 p_s=\Pcal(\tau=s),\qquad
 \kappa_s=\E[D_H\mid\tau=s],
\]
where \(\inf\varnothing=\infty\) and \(\kappa_s=0\) if \(p_s=0\).
Since \(D_H=0\) on \(\{\tau=\infty\}\) and
\(HD_H=1+\sum_{t=s+1}^H I_t\) on \(\{\tau=s\}\),
\begin{equation}
 \Rsmc=\sum_{s=1}^H p_s\kappa_s,\qquad
 \frac1H\le\kappa_s\le\frac{H-s+1}{H}\quad(p_s>0).
 \label{eq:first-divergence-identity}
\end{equation}
The timing law is determined by survival and the conditional mismatch
hazard. Let \(S_s=\Pcal(\tau\ge s)\) and
\(\lambda_s=\E[\delta_s(X_{<s},Y_{<s})\mid\tau\ge s]\), with
\(\lambda_s=0\) when \(S_s=0\). Since
\(\{\tau\ge s\}\in\mathcal F_{s-1}\),
\begin{equation}
 p_s=\E[I_s\one\{\tau\ge s\}]=S_s\lambda_s,\qquad
 S_{s+1}=S_s(1-\lambda_s),\quad S_1=1.
 \label{eq:survival-hazard}
\end{equation}
First mismatches and subsequent occupation contribute
\begin{equation}
 \begin{aligned}
 O_H&=\frac1H\sum_{s=1}^H p_s,\\
 \Pi_H&=\frac1H\sum_{t=1}^H\E[I_t\one\{\tau<t\}]
       =\frac1H\sum_{t=1}^H
          \E[\delta_t(X_{<t},Y_{<t})\one\{\tau<t\}],\\
 \Rsmc&=O_H+\Pi_H.
 \end{aligned}
 \label{eq:occurrence-persistence}
\end{equation}
The equality for \(\Pi_H\) follows from the same measurable partition
and Equation~\eqref{eq:cross-history-tv}.
To separate the available time from the rate of mismatch, let
\begin{equation}
 \begin{aligned}
 E_H&=\frac1H\E\!\left[\sum_{t=1}^H\one\{\tau<t\}\right]
     =\frac1H\sum_{s=1}^H p_s(H-s),\\
 C_H&=\begin{cases}\Pi_H/E_H,&E_H>0,\\0,&E_H=0,\end{cases}
 \qquad \Rsmc=O_H+E_HC_H.
 \end{aligned}
 \label{eq:exposure-accounting}
\end{equation}
Since \(0\le I_t\one\{\tau<t\}\le\one\{\tau<t\}\),
\(0\le\Pi_H\le E_H\). Thus \(E_H\) is expected post-divergence
exposure and \(C_H\in[0,1]\) is the mismatch rate weighted by each
trajectory's number of post-divergence positions.

For two actions \(a,b\), define \(\Delta A=A^{(b)}-A^{(a)}\) and
\(\overline A=(A^{(a)}+A^{(b)})/2\). The algebraic identity
\(E^{(b)}C^{(b)}-E^{(a)}C^{(a)}
=\Delta E\,\overline C+\Delta C\,\overline E\) yields
\begin{equation}
 \Delta\Rsmc=\Delta O_H+
       \underbrace{\Delta E_H\,\overline C_H}_{\text{exposure term}}+
       \underbrace{\Delta C_H\,\overline E_H}_{\text{rate term}}.
 \label{eq:exposure-contrast}
\end{equation}
This identity separates the observed action contrast into occurrence,
exposure, and mismatch-rate terms.

Post-divergence occupation measures mismatch after the first unequal token.
It accommodates later agreements and renewed mismatches: on \(\{\tau\le H\}\),
\[
 \tau<t\le H+1\quad\Longrightarrow\quad
 X_{<t}\ne Y_{<t},
\]
because the unequal coordinate at \(\tau\) remains in each prefix.
On the same event, let \(\sigma=\inf\{t:\tau<t\le H,\ I_t=0\}\), with
\(\sigma=H+1\) when the set is empty. Then
\[
 HD_H=\sigma-\tau+\sum_{t=\sigma+1}^H I_t.
\]
Mismatch occupation can resume after the first later agreement. For
example, deterministic binary sequences \(X=(0,0,0)\) and \(Y=(1,0,1)\)
realize \((I_1,I_2,I_3)=(1,0,1)\).

For a history-independent persistent intervention, set \(P_t=\delta_0\),
\(Q_t=(1-e)\delta_0+e\delta_1\), and \(0<e<1\). The recursively
specified coupling gives \(I_t\stackrel{\mathrm{iid}}{\sim}
\operatorname{Bernoulli}(e)\), so
\[
 \begin{aligned}
 p_s&=(1-e)^{s-1}e,&\qquad \Rsmc&=e,\\
 O_H&=\frac{1-(1-e)^H}{H},&
 \Pi_H&=e-\frac{1-(1-e)^H}{H}.
 \end{aligned}
\]
For \(H\ge2\), independence also gives
\(\Pi_H=eE_H\) and \(C_H=e\). Thus substantial occupation also occurs
under history-independent persistent perturbations.

\subsection{Risks compatible with a finite branch window}
\label{sec:finite-branch-window}

A divergence-aligned window observes the same number of positions after
each entry time, subject to the horizon. Fix \(1\le L\le H\), put
\(\ell_s=\min\{L,H-s+1\}\), and define
\[
 \rho_{s,j}=\Pcal(I_{s+j}=1\mid\tau=s),\qquad
 0\le j<\ell_s,
\]
with \(\rho_{s,j}=0\) when \(p_s=0\). The first mismatch is included,
so \(\rho_{s,0}=1\) for \(p_s>0\). Define the observed contribution
and the remaining capacity by
\begin{equation}
 \begin{aligned}
 R_L(p,\rho)&=\frac1H\sum_{s=1}^H p_s
                         \sum_{j=0}^{\ell_s-1}\rho_{s,j},\\
 B_L(p)&=\frac1H\sum_{s=1}^H p_s(H-s-L+1)_+.
 \end{aligned}
 \label{eq:profile-components}
\end{equation}
In particular, \(R_1=O_H\), \(R_H=\Rsmc\), and \(B_H=0\).
Partial identification asks which values of \(\Rsmc\) are compatible
with these observations~\citep{manski2003partial}.

\begin{definition}[Admissible depth-\(L\) profile]
\label{def:depth-profile}
The profile \(\mathcal S_L(P,Q)=(p,\rho)\) is admissible if
\[
 p_s\ge0,\qquad \sum_{s=1}^H p_s\le1,\qquad
 0\le\rho_{s,j}\le1,
\]
with \(\rho_{s,0}=1\) when \(p_s>0\) and
\(\rho_{s,j}=0\) when \(p_s=0\).
\end{definition}

\begin{theorem}[Sharp identified risk interval]
\label{thm:sharp-frontier}
For every admissible profile on an alphabet with at least two symbols,
allowing \(P,Q\) to vary over all autoregressive kernels subject to that
profile, the compatible risks are exactly
\begin{equation}
 \{\Rsmc(P,Q):\mathcal S_L(P,Q)=(p,\rho)\}
 =[R_L(p,\rho),R_L(p,\rho)+B_L(p)].
 \label{eq:sharp-frontier}
\end{equation}
\end{theorem}

\begin{proof}
For every compatible pair, \(0\le I_t\le1\) implies
\[
 \begin{aligned}
 0\le\Rsmc-R_L
 &=\frac1H\sum_{s:p_s>0}p_s
     \sum_{j=\ell_s}^{H-s}\E[I_{s+j}\mid\tau=s]\\
 &\le\frac1H\sum_s p_s(H-s+1-\ell_s)=B_L.
 \end{aligned}
\]
For attainability, fix \(\alpha\in[0,1]\), use symbols \(0,1\), and set
\(P_t(0\mid x_{<t})=1\). Define
\[
 \bar S_t=1-\sum_{j<t}p_j,\qquad
 \eta_t=\begin{cases}p_t/\bar S_t,&\bar S_t>0,\\0,&\bar S_t=0.\end{cases}
\]
Admissibility gives \(0\le p_t\le\bar S_t\), so \(\eta_t\in[0,1]\).
For a binary history let \(s(y)=\min\{j:y_j=1\}\), with
\(\min\varnothing=\infty\), and define
\[
 Q_t^\alpha(1\mid y_{<t})=
 \begin{cases}
  \eta_t,&s(y)=\infty,\\
  \rho_{s,t-s},&s=s(y)<t,\ t-s<\ell_s,\\
  \alpha,&s=s(y)<t,\ t-s\ge\ell_s.
 \end{cases}
\]
Set \(Q_t^\alpha(0\mid y_{<t})=1-Q_t^\alpha(1\mid y_{<t})\), give
other symbols zero mass, and choose kernels arbitrarily on nonbinary
histories. Under Equation~\eqref{eq:coupling-kernel}, \(X_t=0\) and
\(I_t=Y_t\) almost surely. Induction using
\(\bar S_{t+1}=\bar S_t(1-\eta_t)\) gives
\[
 \Pcal(\tau\ge s)=\prod_{j<s}(1-\eta_j)=\bar S_s,
 \qquad \Pcal(\tau=s)=\bar S_s\eta_s=p_s.
\]
For every \(p_s>0\),
\[
 \Pcal(I_{s+j}=1\mid\tau=s)=
 \begin{cases}
  \rho_{s,j},&0\le j<\ell_s,\\
  \alpha,&\ell_s\le j\le H-s.
 \end{cases}
\]
Consequently, \(\mathcal S_L(P,Q^\alpha)=(p,\rho)\) and
\(\Rsmc(P,Q^\alpha)=R_L+\alpha B_L\). Varying \(\alpha\) fills the
claimed interval, including the singleton case \(B_L=0\).
\end{proof}

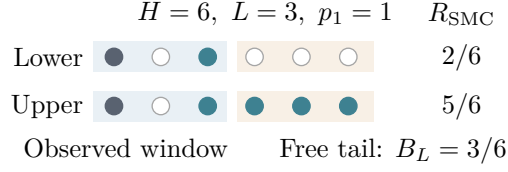
\begin{figure}[t]
\centering
\begin{tikzpicture}[x=1cm,y=1cm,font=\figurefont,line cap=round]
\path[use as bounding box] (0,0) rectangle (6.7,2.05);
\node[inner sep=0pt,] at (3.35000,1.92000) {$H=6,\ L=3,\ p_1=1$};
\node[inner sep=0pt,] at (5.95000,1.92000) {$R_{\rm SMC}$};
\node[inner sep=0pt,anchor=east] at (0.95000,1.35000) {Lower};
\path[fill=figblue!10] (1.08000,1.15000) rectangle (2.85000,1.55000);
\path[fill=figrate!15] (2.98000,1.15000) rectangle (4.78000,1.55000);
\filldraw[fill=figfloor,draw=figfloor,line width=.5pt] (1.35000,1.35000) circle[radius=.115cm];
\filldraw[fill=white,draw=black!35,line width=.5pt] (1.97000,1.35000) circle[radius=.115cm];
\filldraw[fill=figexposure,draw=figexposure,line width=.5pt] (2.59000,1.35000) circle[radius=.115cm];
\filldraw[fill=white,draw=black!35,line width=.5pt] (3.21000,1.35000) circle[radius=.115cm];
\filldraw[fill=white,draw=black!35,line width=.5pt] (3.83000,1.35000) circle[radius=.115cm];
\filldraw[fill=white,draw=black!35,line width=.5pt] (4.45000,1.35000) circle[radius=.115cm];
\node[inner sep=0pt,] at (5.95000,1.35000) {$2/6$};
\node[inner sep=0pt,anchor=east] at (0.95000,0.70000) {Upper};
\path[fill=figblue!10] (1.08000,0.50000) rectangle (2.85000,0.90000);
\path[fill=figrate!15] (2.98000,0.50000) rectangle (4.78000,0.90000);
\filldraw[fill=figfloor,draw=figfloor,line width=.5pt] (1.35000,0.70000) circle[radius=.115cm];
\filldraw[fill=white,draw=black!35,line width=.5pt] (1.97000,0.70000) circle[radius=.115cm];
\filldraw[fill=figexposure,draw=figexposure,line width=.5pt] (2.59000,0.70000) circle[radius=.115cm];
\filldraw[fill=figexposure,draw=figexposure,line width=.5pt] (3.21000,0.70000) circle[radius=.115cm];
\filldraw[fill=figexposure,draw=figexposure,line width=.5pt] (3.83000,0.70000) circle[radius=.115cm];
\filldraw[fill=figexposure,draw=figexposure,line width=.5pt] (4.45000,0.70000) circle[radius=.115cm];
\node[inner sep=0pt,] at (5.95000,0.70000) {$5/6$};
\node[inner sep=0pt,] at (3.35000,0.12000) {Observed window \quad\; Free tail: $B_L=3/6$};
\end{tikzpicture}
\caption{A shared observation window admits different cumulative risks.
Filled circles indicate mismatches. Both paths have the same depth-three
profile; agreement or mismatch throughout the remaining three positions
attains the lower or upper endpoint of Theorem~\ref{thm:sharp-frontier}.}
\label{fig:depth-frontier}
\end{figure}

The capacity \(B_L\) quantifies unobserved occupation; the actual tail
\(R_H-R_L\) locates an intervention within the interval
(Figure~\ref{fig:depth-frontier}).

\subsection{Consequences of the identified interval}
\label{sec:interval-consequences}

\begin{corollary}[Exact-profile minimax estimation]
\label{cor:profile-minimax}
For a fixed admissible profile, write \(r=R_L\), \(b=B_L\). Then
\[
 \inf_T\sup_{u\in[r,r+b]}|T-u|=\frac b2,\qquad
 \inf_T\sup_{u\in[r,r+b]}(T-u)^2=\frac{b^2}{4},
\]
with both infima attained by \(T=r+b/2\).
\end{corollary}
\begin{proof}
Theorem~\ref{thm:sharp-frontier} attains both endpoints at the same profile.
For every \(T\in\mathbb R\),
\[
 \begin{aligned}
 \max\{|T-r|,|T-r-b|\}&\ge\frac{|T-r|+|T-r-b|}{2}\ge\frac b2,\\
 \max\{(T-r)^2,(T-r-b)^2\}
 &\ge\frac{(T-r)^2+(T-r-b)^2}{2}\\
 &=\left(T-r-\frac b2\right)^2+\frac{b^2}{4}.
 \end{aligned}
\]
At \(T=r+b/2\), \(|T-u|\le b/2\) for every \(u\in[r,r+b]\).
\end{proof}

\begin{corollary}[Shared-reference action ordering]
\label{cor:action-ordering}
For two admissible profiles at the same \(H,L\), let
\([\underline r_c,\overline r_c]\) be their identified intervals,
\(c\in\{a,b\}\). Allowing a shared reference \(P\) and both action
kernels to vary subject to those profiles, the sharp contrast set is
\[
 \{\Rsmc(P,Q^b)-\Rsmc(P,Q^a)\}
 =[\underline r_b-\overline r_a,\overline r_b-\underline r_a].
\]
Every compatible triple satisfies \(\Rsmc(P,Q^a)<\Rsmc(P,Q^b)\)
if and only if \(\overline r_a<\underline r_b\).
\end{corollary}
\begin{proof}
Write \(\mathcal S_L^{(c)}\) for the specified profile of action \(c\),
\(I_c=[\underline r_c,\overline r_c]\), and
\(b_c=\overline r_c-\underline r_c=B_L^{(c)}\). By
Theorem~\ref{thm:sharp-frontier}, the joint risk set satisfies
\[
 \mathcal J:=\left\{\bigl(\Rsmc(P,Q^a),\Rsmc(P,Q^b)\bigr):
 \mathcal S_L(P,Q^c)=\mathcal S_L^{(c)},\ c=a,b\right\}
 \subseteq I_a\times I_b.
\]
For \(P_t^\star(\cdot\mid h)=\delta_0\), the same theorem's construction gives,
for every \((\alpha_a,\alpha_b)\in[0,1]^2\),
\[
 \begin{aligned}
 \mathcal S_L(P^\star,Q^{c,\alpha_c})&=\mathcal S_L^{(c)},\qquad c=a,b,\\
 \bigl(\Rsmc(P^\star,Q^{a,\alpha_a}),\Rsmc(P^\star,Q^{b,\alpha_b})\bigr)
 &=(\underline r_a+\alpha_a b_a,\underline r_b+\alpha_b b_b).
 \end{aligned}
\]
Hence \(I_a\times I_b\subseteq\mathcal J\), and
\[
 \begin{aligned}
 \{v-u:(u,v)\in\mathcal J\}
 &=\underline r_b-\underline r_a+[0,b_b]-[0,b_a]\\
 &=[\underline r_b-\overline r_a,\overline r_b-\underline r_a],\\
 \bigl[\forall (u,v)\in\mathcal J,\ u<v\bigr]
 &\iff\min_{(u,v)\in\mathcal J}(v-u)>0
 \iff\overline r_a<\underline r_b.
 \end{aligned}
\]
\end{proof}

\subsection{Additional information from a fixed reference}
\label{sec:fixed-reference}

Specifying the complete reference kernel \(P_0\) restricts the compatible
class:
\[
 \mathcal I_L(P_0;p,\rho)
 :=\{\Rsmc(P_0,Q):\mathcal S_L(P_0,Q)=(p,\rho)\}
 \subseteq[R_L,R_L+B_L].
\]
The inclusion can be strict. Let \(\mathcal V=\{0,1\}\), \(H=2\),
\(L=1\), \(P_{0,t}(1\mid h)=1/2\) at every history, and
\(p_1=\varepsilon\in(0,1/2)\), \(p_2=0\). Necessarily
\[
 Q_1(1)=\tfrac12\pm\varepsilon,\qquad
 c_1(z)=\min\{\tfrac12,Q_1(z)\}>0\quad(z=0,1).
\]
The probability of first divergence at time two satisfies
\[
 0=p_2=\sum_{z=0}^1 c_1(z)
      \TV\!\left(\operatorname{Bernoulli}(\tfrac12),Q_2(\cdot\mid z)\right).
\]
Each summand is nonnegative with positive coefficient, whence
\(Q_2(\cdot\mid z)=\operatorname{Bernoulli}(1/2)\) for both histories.
Thus \(\delta_2(x,y)=0\) for every \(x,y\in\{0,1\}\), and
\[
 \mathcal I_1(P_0;p,\rho)=\{\varepsilon/2\}
 \subsetneq[\varepsilon/2,\varepsilon]=[R_1,R_1+B_1].
\]
Either choice of \(Q_1(1)=1/2\pm\varepsilon\), followed by fair
second-step kernels, attains the singleton. For any specified reference,
separated general intervals remain a sufficient action-order condition.

\section{Measuring Occurrence and Subsequent Occupation}
\label{sec:swrb}

Residual-branch conditional Monte Carlo measures the decomposition by
separating branch probability from within-branch loss. Survival-Weighted Residual Branching (SWRB) integrates out the zero-loss
full-agreement event. The same sampled suffix then provides the observed
window and the realized tail, retaining their covariance. We suppress
the action index in \(P,Q^a\).

\subsection{Conditioning on a residual branch}
\label{sec:residual-conditioning}

At a common prefix \(h\), let
\(c_s(z\mid h)=\min\{P_s(z\mid h),Q_s(z\mid h)\}\) and
\(e_s(h)=1-\sum_z c_s(z\mid h)\). Draw an auxiliary path \(U\) using
\begin{equation}
 K_s^\cap(z\mid h)=\frac{c_s(z\mid h)}{1-e_s(h)}\quad(e_s(h)<1),
 \label{eq:overlap-kernel}
\end{equation}
extended arbitrarily when \(e_s=1\). Its survival weights and branch mass are
\begin{equation}
 \widetilde S_s=\prod_{j<s}(1-e_j(U_{<j})),\qquad
 w_s=\widetilde S_s e_s(U_{<s}),\qquad
 M=\sum_s w_s=1-\widetilde S_{H+1}.
 \label{eq:branch-weights}
\end{equation}
The overlap identity \((1-e_s)K_s^\cap=c_s\) gives
\begin{equation}
 \Pcal_U(U_{<s}=u_{<s})w_s(u)
 =\Pcal(X_{<s}=Y_{<s}=u_{<s},\tau=s).
 \label{eq:prefix-mixture}
\end{equation}
Thus \(\E[w_s]=p_s\). An auxiliary mixture conditional on \(U\) assigns zero-loss mass
\(1-M\) and residual-branch masses \(w_s\); marginalizing \(U\)
recovers the ordinary coupled-loss law. A branch at \(s\) draws the residual token pair at \(U_{<s}\)
and continues the declared coupling with fresh randomness. Write its loss
as \(D_s=H^{-1}\sum_{t=s}^H I_t\), and its conditional second moment
as \(\nu_s=\E[D_s^2\mid U]\), zero for a zero-weight branch.

We remove the zero-loss component by sampling \(J\mid U\) with
probabilities \(w_s/M\) for \(M>0\) and generating the selected suffix.
If an auxiliary
\(B\mid U\sim\operatorname{Bernoulli}(M)\) is independent of
\((J,D_J)\) given \(U\), then \(BD_J\) has the ordinary coupled-loss
law. Conditional expectation gives
\begin{equation}
 \E[BD_J\mid U,J,D_J]=MD_J.
 \label{eq:conditional-estimator}
\end{equation}
This is the Rao--Blackwell representation
\citep{blackwell1947conditional,casella1996raoblackwellisation,robert2004montecarlo}.
At \(M=0\), take \(J=1,D_J=B=0\).

\begin{proposition}[Unbiased measurement and per-replicate variance dominance]
\label{thm:swrb}
The estimates
\begin{equation}
 Z^R=MD_J,\qquad Z^O=M/H,\qquad Z^\Pi=M(D_J-1/H)
 \label{eq:swrb-components}
\end{equation}
satisfy \(Z^R=Z^O+Z^\Pi\) and have expectations
\(\Rsmc,O_H,\Pi_H\), respectively. For ordinary coupled loss
\(Z_{\mathrm{naive}}\),
\begin{equation}
 \operatorname{Var}(Z_{\mathrm{naive}})-\operatorname{Var}(Z^R)
 =\E\!\left[(1-M)\sum_s w_s\nu_s\right]\ge0,
 \label{eq:swrb-gap}
\end{equation}
with strict inequality exactly when \(\Pcal(0<M<1)>0\).
\end{proposition}

\begin{proof}
Let \(m_s=\E[D_s\mid U]\), zero on zero-weight branches. By
Equation~\eqref{eq:prefix-mixture}, \(\E[w_sm_s]=p_s\kappa_s\).
The branch law and conditional independence of \(B\) give
\[
 \begin{aligned}
 \E[BD_J\mid U]&=\E[MD_J\mid U]=\sum_s w_sm_s,\\
 \E[(BD_J)^2\mid U]&=\sum_s w_s\nu_s,\qquad
 \E[(MD_J)^2\mid U]=M\sum_s w_s\nu_s.
 \end{aligned}
\]
Thus \(\E[Z^R]=\sum_s p_s\kappa_s=\Rsmc\),
\(\E[Z^O]=H^{-1}\sum_s p_s=O_H\), and
\(\E[Z^\Pi]=\E[Z^R-Z^O]=\Pi_H\).
Since \(BD_J\overset{d}=Z_{\mathrm{naive}}\), subtracting second moments
proves Equation~\eqref{eq:swrb-gap}. On \(\{w_s>0\}\),
\(H^{-1}\le D_s\le1\); hence
\[
 \frac{M(1-M)}{H^2}\le(1-M)\sum_s w_s\nu_s\le M(1-M).
\]
The expected gap is positive if and only if \(\Pcal(0<M<1)>0\).
All identities include \(M=0\) under the stated zero convention.
\end{proof}

\subsection{Joint measurement of the window and tail}
\label{sec:window-tail-measurement}

The branch-mixture identity also applies to a truncated suffix, so a
single replicate can resolve how loss accumulates with observation depth.
For \(\ell_s=\min\{L,H-s+1\}\), define
\(D_{s,L}=H^{-1}\sum_{j=0}^{\ell_s-1}I_{s+j}\),
\(m_{s,L}=\E[D_{s,L}\mid U]\), \(Z_L=MD_{J,L}\), and
\(W_L=H^{-1}\sum_s w_s(H-s-L+1)_+\). Applying the same mixture to
the window contribution, with \(D_{s,L}=m_{s,L}=0\) when \(w_s=0\), gives
\begin{equation}
 \E[Z_L]=R_L,\qquad \E[Z_H-Z_L]=R_H-R_L,\qquad \E[W_L]=B_L.
 \label{eq:depth-measurement}
\end{equation}
Indeed, \(\E[Z_L\mid U]=\sum_s w_s m_{s,L}\), and
\(\E[w_s m_{s,L}]=(p_s/H)\sum_{j<\ell_s}\rho_{s,j}\).
Summation and linearity prove Equation~\eqref{eq:depth-measurement}.
The shared suffix preserves the covariance of the window and tail estimates
(Figure~\ref{fig:swrb-measurement}, Algorithm~\ref{alg:conditional-measurement}).

\begin{figure}[H]
\centering
\begin{subfigure}[b]{.32\linewidth}
\centering\begin{tikzpicture}[x=1cm,y=1cm,font=\figurefont,line cap=round]
\path[use as bounding box] (0,0) rectangle (4.35,1.65);
\node[inner sep=0pt,] at (2.17500,1.50000) {$H=5,\quad M=0.30$};
\path[fill=black!8] (0.10000,0.65000) rectangle (2.90000,1.08000);
\node[inner sep=0pt,] at (1.50000,0.86500) {Agreement: 0.70};
\path[fill=figblue!50] (2.90000,0.65000) rectangle (3.02000,1.08000);
\path[fill=figblue!65] (3.02000,0.65000) rectangle (3.26000,1.08000);
\path[fill=figred] (3.26000,0.65000) rectangle (3.62000,1.08000);
\path[fill=figblue!80] (3.62000,0.65000) rectangle (3.92000,1.08000);
\path[fill=figblue] (3.92000,0.65000) rectangle (4.10000,1.08000);
\node[inner sep=0pt,] at (2.17500,0.22000) {$M=\sum_s w_s$};
\end{tikzpicture}
\caption{Scan branch mass}
\end{subfigure}\hfill
\begin{subfigure}[b]{.32\linewidth}
\centering\begin{tikzpicture}[x=1cm,y=1cm,font=\figurefont,line cap=round]
\path[use as bounding box] (0,0) rectangle (4.35,1.65);
\node[inner sep=0pt,] at (2.17500,1.50000) {$J\mid U\sim w/M$};
\path[fill=figblue!50] (0.15000,0.65000) rectangle (0.55000,1.08000);
\node[inner sep=0pt,text=white] at (0.35000,0.86500) {1};
\path[fill=figblue!65] (0.55000,0.65000) rectangle (1.35000,1.08000);
\node[inner sep=0pt,text=white] at (0.95000,0.86500) {2};
\path[fill=figred] (1.35000,0.65000) rectangle (2.55000,1.08000);
\node[inner sep=0pt,text=white] at (1.95000,0.86500) {3};
\path[fill=figblue!80] (2.55000,0.65000) rectangle (3.55000,1.08000);
\node[inner sep=0pt,text=white] at (3.05000,0.86500) {4};
\path[fill=figblue] (3.55000,0.65000) rectangle (4.15000,1.08000);
\node[inner sep=0pt,text=white] at (3.85000,0.86500) {5};
\node[inner sep=0pt,] at (2.17500,0.22000) {Sample $J=3$: probability $0.30$};
\end{tikzpicture}
\caption{Select one branch}
\end{subfigure}\hfill
\begin{subfigure}[b]{.32\linewidth}
\centering\begin{tikzpicture}[x=1cm,y=1cm,font=\figurefont,line cap=round]
\path[use as bounding box] (0,0) rectangle (4.35,1.65);
\path[fill=figexposure!10] (1.90000,1.10000) rectangle (4.20000,1.65000);
\filldraw[fill=white,draw=black!35,line width=.5pt] (0.55000,1.43000) circle[radius=.115cm];
\node[inner sep=0pt,] at (0.55000,1.02000) {1};
\filldraw[fill=white,draw=black!35,line width=.5pt] (1.40000,1.43000) circle[radius=.115cm];
\node[inner sep=0pt,] at (1.40000,1.02000) {2};
\filldraw[fill=figfloor,draw=figfloor,line width=.5pt] (2.25000,1.43000) circle[radius=.115cm];
\node[inner sep=0pt,] at (2.25000,1.02000) {3};
\filldraw[fill=white,draw=black!35,line width=.5pt] (3.10000,1.43000) circle[radius=.115cm];
\node[inner sep=0pt,] at (3.10000,1.02000) {4};
\filldraw[fill=figexposure,draw=figexposure,line width=.5pt] (3.95000,1.43000) circle[radius=.115cm];
\node[inner sep=0pt,] at (3.95000,1.02000) {5};
\draw[figfloor,->,line width=.5pt] (2.25000,0.87000) -- (2.25000,0.77000) -- (1.10000,0.77000) -- (1.10000,0.62000);
\draw[figexposure,->,line width=.5pt] (3.95000,0.87000) -- (3.95000,0.71000) -- (2.02000,0.71000) -- (2.02000,0.62000);
\path[fill=figfloor] (0.64000,0.18000) rectangle (1.56000,0.62000);
\path[fill=figexposure] (1.56000,0.18000) rectangle (2.48000,0.62000);
\node[inner sep=0pt,text=white] at (1.10000,0.40000) {0.06};
\node[inner sep=0pt,text=white] at (2.02000,0.40000) {0.06};
\node[inner sep=0pt,] at (3.43000,0.40000) {$Z^R=0.12$};
\end{tikzpicture}
\caption{Weight its loss}
\end{subfigure}
\caption{One residual-branch measurement, conditional on an overlap scan.
The scan places mass \(M=0.30\) on divergence. Renormalizing its branch
weights selects one suffix; the pictured sample has \(D_3=2/5\) and
returns \(Z^R=0.12\). In (c), filled circles mark mismatches at positions
3 and 5; their bars contribute \(Z^O=Z^\Pi=M/H=0.06\), respectively.
The factor \(M\) preserves the unconditional mean.}
\label{fig:swrb-measurement}
\end{figure}
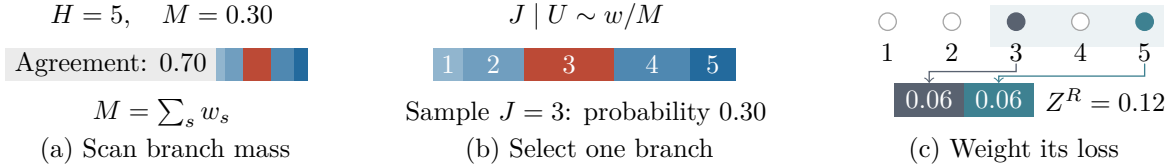

\begin{algorithm}[H]
\caption{One conditional Monte Carlo replicate}
\label{alg:conditional-measurement}
\begin{algorithmic}[1]
\Require Kernels \(P,Q\), horizon \(H\), depths \(\mathcal L\subseteq\{1,\ldots,H\}\)
\Ensure \((Z^R,Z^O,Z^\Pi)\) and \((Z_L,Z_H-Z_L,W_L)_{L\in\mathcal L}\)
\State Initialize \(u\gets()\), \(S\gets1\), and \(w_1,\ldots,w_H\gets0\).
\For{\(s=1,\ldots,H\)}
 \State Compute \(e_s(u)\); set \(w_s\gets S e_s(u)\), \(S\gets S(1-e_s(u))\).
 \State If \(S=0\), \textbf{break}; all later branch weights are zero.
 \State Draw \(U_s\sim K_s^\cap(\cdot\mid u)\) and append it to \(u\).
\EndFor
\State Set \(M\gets\sum_s w_s\), \(W_L\gets H^{-1}\sum_s w_s(H-s-L+1)_+\).
\State If \(M=0\), \Return zero for every output.
\State Draw \(J\) with probabilities \(w_s/M\).
\State Draw a fresh residual pair at \(U_{<J}\); continue the declared coupling, recording \(I_J,\ldots,I_H\).
\State Compute \(Z_L\gets (M/H)\sum_{j=0}^{\min(L,H-J+1)-1}I_{J+j}\), including \(L=H\).
\State \Return \((Z_H,M/H,Z_H-M/H)\), \((Z_L,Z_H-Z_L,W_L)_{L\in\mathcal L}\).
\end{algorithmic}
\end{algorithm}

The cost includes the entire overlap scan, retention or replay of \(U_{<J}\),
and residual continuation. Wall-clock efficiency depends on these costs
together with the per-replicate variance.

\subsection{Sampled profiles and endpoint uncertainty}
\label{sec:sampled-profile}

The joint measurements \(Z_L,W_L\) in
Equation~\eqref{eq:depth-measurement} satisfy
\[
 0\le Z_L\le\frac{ML}{H},\qquad
 0\le W_L\le\frac{M(H-L)}{H},\qquad
 0\le Z_L\le Z_L+W_L\le1.
\]
Let \(V_i^-\) and \(V_i^+\) be within-document replicate averages of
\(Z_L\) and \(Z_L+W_L\), respectively. For \(n\) independent document
vectors, define
\[
 \widehat r_\pm=\frac1n\sum_{i=1}^n V_i^\pm,\qquad
 r_\pm=\E[\widehat r_\pm],\qquad
 \epsilon_\alpha=\sqrt{\frac{\log(4/\alpha)}{2n}}.
\]
Hoeffding's inequality~\citep{hoeffding1963probability} yields
\(\Pcal(|\widehat r_\pm-r_\pm|>\epsilon_\alpha)
\le2e^{-2n\epsilon_\alpha^2}=\alpha/2\).
By a union bound, the corresponding mean full-horizon risk belongs to
\[
 \left[\max\{0,\widehat r_- -\epsilon_\alpha\},\quad
       \min\{1,\widehat r_+ +\epsilon_\alpha\}\right]
\]
with probability at least \(1-\alpha\) for each specified action and depth.
The enclosure combines unobserved tail capacity and sampled-endpoint
uncertainty.

The finite-state study below evaluates SWRB. The language-model studies
in Section~\ref{sec:experiments} measure the same \(\Rsmc,O_H,\Pi_H\)
with complete coupled rollouts for temporal analysis.

\subsection{Finite-state estimation}
\label{sec:measurement-validation}

We compare the conditional estimator with ordinary coupling in 24
enumerable binary systems, using \(H=3\) and 512 independent replicates
per estimator and system. For each system $c\in\{1,\ldots,24\}$, every prefix has independent
distributions $P,A$, obtained by normalizing independent
$\operatorname{Gamma}(1.5,1)+0.05$ weights. The intervention kernel is
\[
 Q=(1-\lambda)P+\lambda A,\qquad
 \lambda=\left(0.05+\frac{0.45c}{24}\right)UV,\qquad
 U\sim\operatorname{Unif}[0.5,1],\quad V\sim\operatorname{Unif}[0,1],
\]
with independent $U,V$. Exact enumeration at $H=3$ supplies
$\Rsmc$ for each system.

Figure~\ref{fig:swrb-validation} compares the sample means with exact
risks and shows lower SWRB sample variance in every system. The
naive-to-SWRB sample-variance ratio ranges from
\(11.1\) to \(289\), with median \(52.7\), consistent with
Proposition~\ref{thm:swrb}.

\begin{figure}[htbp]
\centering
\begin{subfigure}[t]{.485\linewidth}
\centering\begin{tikzpicture}[x=1cm,y=1cm,font=\figurefont,line cap=round]
\path[use as bounding box] (0,-.3) rectangle (7.5,4.4);
\draw[black!12] (1.35000,.75)--(1.35000,3.65);
\draw[black!12] (1.35000,0.75000)--(7.00000,0.75000);
\node[anchor=north] at (1.35000,.64) {0};
\node[anchor=east] at (1.25424,0.75000) {0};
\draw[black!12] (3.08846,.75)--(3.08846,3.65);
\draw[black!12] (1.35000,1.64231)--(7.00000,1.64231);
\node[anchor=north] at (3.08846,.64) {0.04};
\node[anchor=east] at (1.25424,1.64231) {0.04};
\draw[black!12] (4.82692,.75)--(4.82692,3.65);
\draw[black!12] (1.35000,2.53462)--(7.00000,2.53462);
\node[anchor=north] at (4.82692,.64) {0.08};
\node[anchor=east] at (1.25424,2.53462) {0.08};
\draw[black!12] (6.56538,.75)--(6.56538,3.65);
\draw[black!12] (1.35000,3.42692)--(7.00000,3.42692);
\node[anchor=north] at (6.56538,.64) {0.12};
\node[anchor=east] at (1.25424,3.42692) {0.12};
\draw[black!65] (1.35000,3.65)--(1.35000,.75)--(7.00000,.75);
\draw[black!55,dashed] (1.35000,.75)--(7.00000,3.65);
\node at (4.17500,.07) {Exact $R_{\mathrm{SMC}}$};
\node[rotate=90] at (.12,2.2) {Estimated $R_{\mathrm{SMC}}$};
\draw[black!65,fill=white,line width=.55pt] (1.97643,0.92428) circle (1.8pt);
\fill[figblue] (1.97643,1.07730) circle (1.3pt);
\draw[black!65,fill=white,line width=.55pt] (2.32051,1.35998) circle (1.8pt);
\fill[figblue] (2.32051,1.25550) circle (1.3pt);
\draw[black!65,fill=white,line width=.55pt] (1.84788,0.96785) circle (1.8pt);
\fill[figblue] (1.84788,1.00416) circle (1.3pt);
\draw[black!65,fill=white,line width=.55pt] (1.47019,0.83714) circle (1.8pt);
\fill[figblue] (1.47019,0.80996) circle (1.3pt);
\draw[black!65,fill=white,line width=.55pt] (3.01218,1.53425) circle (1.8pt);
\fill[figblue] (3.01218,1.59730) circle (1.3pt);
\draw[black!65,fill=white,line width=.55pt] (2.04255,1.04046) circle (1.8pt);
\fill[figblue] (2.04255,1.11365) circle (1.3pt);
\draw[black!65,fill=white,line width=.55pt] (2.78198,1.57782) circle (1.8pt);
\fill[figblue] (2.78198,1.49271) circle (1.3pt);
\draw[black!65,fill=white,line width=.55pt] (2.00632,1.02594) circle (1.8pt);
\fill[figblue] (2.00632,1.09287) circle (1.3pt);
\draw[black!65,fill=white,line width=.55pt] (2.94867,1.69401) circle (1.8pt);
\fill[figblue] (2.94867,1.56577) circle (1.3pt);
\draw[black!65,fill=white,line width=.55pt] (2.53842,1.35998) circle (1.8pt);
\fill[figblue] (2.53842,1.33157) circle (1.3pt);
\draw[black!65,fill=white,line width=.55pt] (2.91341,1.66496) circle (1.8pt);
\fill[figblue] (2.91341,1.53327) circle (1.3pt);
\draw[black!65,fill=white,line width=.55pt] (3.49764,2.01352) circle (1.8pt);
\fill[figblue] (3.49764,1.85926) circle (1.3pt);
\draw[black!65,fill=white,line width=.55pt] (3.02034,1.47616) circle (1.8pt);
\fill[figblue] (3.02034,1.65098) circle (1.3pt);
\draw[black!65,fill=white,line width=.55pt] (2.96665,1.44712) circle (1.8pt);
\fill[figblue] (2.96665,1.58062) circle (1.3pt);
\draw[black!65,fill=white,line width=.55pt] (3.75239,1.82472) circle (1.8pt);
\fill[figblue] (3.75239,1.96616) circle (1.3pt);
\draw[black!65,fill=white,line width=.55pt] (5.20422,2.75421) circle (1.8pt);
\fill[figblue] (5.20422,2.72299) circle (1.3pt);
\draw[black!65,fill=white,line width=.55pt] (3.84500,2.10066) circle (1.8pt);
\fill[figblue] (3.84500,2.10781) circle (1.3pt);
\draw[black!65,fill=white,line width=.55pt] (2.72150,1.49069) circle (1.8pt);
\fill[figblue] (2.72150,1.45225) circle (1.3pt);
\draw[black!65,fill=white,line width=.55pt] (4.41422,2.27494) circle (1.8pt);
\fill[figblue] (4.41422,2.36832) circle (1.3pt);
\draw[black!65,fill=white,line width=.55pt] (2.52906,1.46164) circle (1.8pt);
\fill[figblue] (2.52906,1.35970) circle (1.3pt);
\draw[black!65,fill=white,line width=.55pt] (4.87579,2.68159) circle (1.8pt);
\fill[figblue] (4.87579,2.58292) circle (1.3pt);
\draw[black!65,fill=white,line width=.55pt] (2.55535,1.28736) circle (1.8pt);
\fill[figblue] (2.55535,1.35519) circle (1.3pt);
\draw[black!65,fill=white,line width=.55pt] (6.61360,3.58203) circle (1.8pt);
\fill[figblue] (6.61360,3.46820) circle (1.3pt);
\draw[black!65,fill=white,line width=.55pt] (4.31090,2.21685) circle (1.8pt);
\fill[figblue] (4.31090,2.26910) circle (1.3pt);
\draw[black!65,fill=white] (1.68517,4.08) circle (1.8pt);
\node[anchor=west] at (1.87669,4.08) {Naive};
\fill[figblue] (3.55254,4.08) circle (1.3pt);
\node[anchor=west] at (3.74407,4.08) {SWRB};
\draw[black!55,dashed] (5.08475,4.08)--(5.46780,4.08);
\node[anchor=west] at (5.54441,4.08) {Identity};
\end{tikzpicture}
\caption{Estimates and exact risks}
\end{subfigure}\hfill
\begin{subfigure}[t]{.485\linewidth}
\centering\begin{tikzpicture}[x=1cm,y=1cm,font=\figurefont,line cap=round]
\path[use as bounding box] (0,-.3) rectangle (7.5,4.4);
\draw[black!12] (1.35000,0.75000)--(7.00000,0.75000);
\node[anchor=east] at (1.25424,0.75000) {$10^{0}$};
\draw[black!12] (1.35000,1.88991)--(7.00000,1.88991);
\node[anchor=east] at (1.25424,1.88991) {$10^{1}$};
\draw[black!12] (1.35000,3.02981)--(7.00000,3.02981);
\node[anchor=east] at (1.25424,3.02981) {$10^{2}$};
\node[anchor=north] at (1.35000,.64) {1};
\node[anchor=north] at (2.57826,.64) {6};
\node[anchor=north] at (4.05217,.64) {12};
\node[anchor=north] at (5.52609,.64) {18};
\node[anchor=north] at (7.00000,.64) {24};
\draw[black!65] (1.35000,3.65)--(1.35000,.75)--(7.00000,.75);
\node at (4.17500,.07) {Finite-state cell (sorted)};
\node[rotate=90] at (.13,2.2) {Naive/SWRB variance};
\draw[figrate,dashed] (1.35000,2.71282)--(7.00000,2.71282);
\node at (4.17500,4.08) {24/24 cells; median $52.7\times$};
\draw[figblue,line width=.65pt] (1.35000,1.94170)--(1.59565,2.05523)--(1.84130,2.09017)--(2.08696,2.22618)--(2.33261,2.27352)--(2.57826,2.28881)--(2.82391,2.48470)--(3.06956,2.48811)--(3.31521,2.54457)--(3.56087,2.56936)--(3.80652,2.59798)--(4.05217,2.70845)--(4.29783,2.71716)--(4.54348,2.76298)--(4.78913,2.76694)--(5.03479,2.83726)--(5.28044,2.91209)--(5.52609,2.93112)--(5.77174,3.01374)--(6.01739,3.06581)--(6.26304,3.16603)--(6.50870,3.21395)--(6.75435,3.35238)--(7.00000,3.55566);
\fill[figblue] (1.35000,1.94170) circle (1.5pt);
\fill[figblue] (1.59565,2.05523) circle (1.5pt);
\fill[figblue] (1.84130,2.09017) circle (1.5pt);
\fill[figblue] (2.08696,2.22618) circle (1.5pt);
\fill[figblue] (2.33261,2.27352) circle (1.5pt);
\fill[figblue] (2.57826,2.28881) circle (1.5pt);
\fill[figblue] (2.82391,2.48470) circle (1.5pt);
\fill[figblue] (3.06956,2.48811) circle (1.5pt);
\fill[figblue] (3.31521,2.54457) circle (1.5pt);
\fill[figblue] (3.56087,2.56936) circle (1.5pt);
\fill[figblue] (3.80652,2.59798) circle (1.5pt);
\fill[figblue] (4.05217,2.70845) circle (1.5pt);
\fill[figblue] (4.29783,2.71716) circle (1.5pt);
\fill[figblue] (4.54348,2.76298) circle (1.5pt);
\fill[figblue] (4.78913,2.76694) circle (1.5pt);
\fill[figblue] (5.03479,2.83726) circle (1.5pt);
\fill[figblue] (5.28044,2.91209) circle (1.5pt);
\fill[figblue] (5.52609,2.93112) circle (1.5pt);
\fill[figblue] (5.77174,3.01374) circle (1.5pt);
\fill[figblue] (6.01739,3.06581) circle (1.5pt);
\fill[figblue] (6.26304,3.16603) circle (1.5pt);
\fill[figblue] (6.50870,3.21395) circle (1.5pt);
\fill[figblue] (6.75435,3.35238) circle (1.5pt);
\fill[figblue] (7.00000,3.55566) circle (1.5pt);
\end{tikzpicture}
\caption{Empirical variance reduction}
\end{subfigure}
\caption{Conditional measurements recover finite-state risk with lower
sample variance in all 24 systems. Left: Monte Carlo means from 512
replicates for each estimator and system. Right: ratios of sample variances,
sorted by magnitude; the dashed line marks the median.}
\label{fig:swrb-validation}
\end{figure}
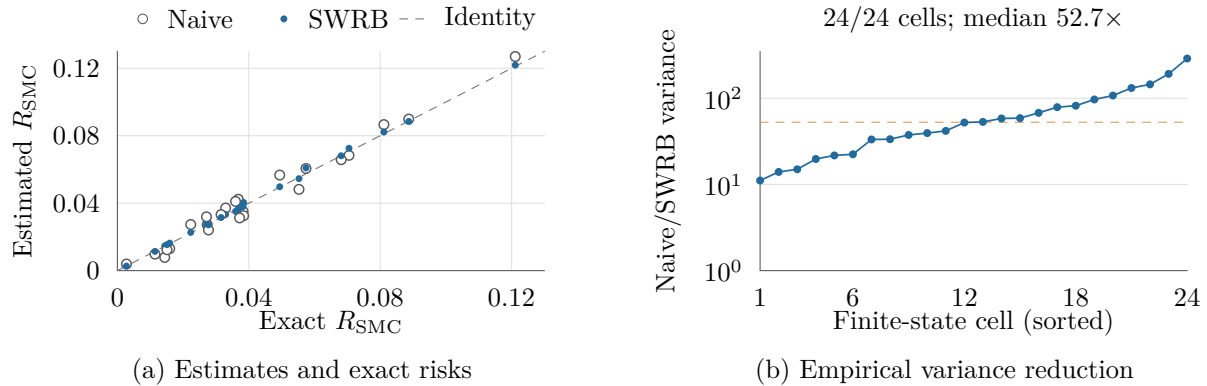

\begin{samepage}
Counting the overlap scan and residual continuation, the mean number of
kernel-pair evaluations rises by a factor of \(1.37\), while the
ratio-of-means variance--cost advantage is \(20.0\). These counts
quantify the computational trade-off in the finite-state systems;
Transformer execution also depends on cache retention and replay.
\end{samepage}

\section{How Eviction Shapes the Disagreement Trajectory}
\label{sec:experiments}
\label{sec:results}

Paired complete trajectories connect the temporal decomposition to
KV-cache eviction in two language models. We examine entry into
divergence, subsequent branch dynamics, and the information retained
by shorter observations.

\subsection{Paired trajectories and study design}
\label{sec:study}

We study Meta-Llama-3.1-8B-Instruct (Llama) and Qwen2.5-7B-Instruct (Qwen)
in six HELMET/RULER
environment--length strata (4K, 16K, and 32K)
\citep{yen2024helmet,hsieh2024ruler}. The complete-trajectory study uses
288 validation documents, 48 per stratum, and \(K=64\) replicates per
action. Actions retain 512 prompt entries with SnapKV, 50\% with SnapKV,
or the same prompt-cache bytes with a recent-token selector. The targeted higher-retention
study uses another 288 independent documents, \(K=32\), and SnapKV versus
recent-token retention at 90\%. Both studies have horizon \(H=128\),
temperature-one full-softmax sampling, persistent compact caches, and
same-path full-retention controls (\(K=2\)); these controls yield zero TV.
Generated entries are retained, and EOS is absorbing through the horizon.

The HELMET strata use retrieved-context and citation tasks from NQ,
TriviaQA, HotpotQA, PopQA, ASQA, and QAMPARI; the citation prompts use
HELMET's top-2000 retrieval files. RULER contributes its 13 synthetic
tasks: eight needle-retrieval variants, variable tracking, common- and
frequent-word extraction, and two question-answering tasks. Length
labels denote nominal input bands; retention uses the actual tokenized
prompt length.

For tokenized prompt length $n$, half and 90\% retention preserve
$\lfloor n/2\rfloor$ and $\lfloor9n/10\rfloor$ prompt entries per KV
head and layer. Recent retention keeps four initial sink positions and
the most recent remaining entries. At the same retention, SnapKV and recent
retention match exactly in prompt-KV payload, indices, metadata, and initial
resident decode bytes. SnapKV uses a 64-token
scoring window and kernel size five. Eviction follows prefill, preserving
the already computed first-token distribution.

For each action and replicate, we record the reference along its own
history \(X\), the compressed model along its own history \(Y\), and a
compressed copy forced along \(X\). Besides token mismatch, these give
\begin{equation}
 \begin{aligned}
 R(a)&=H^{-1}\sum_t\E_{\Gamma(P,Q^a)}
   \TV\!\left(P_t(\cdot\mid X_{<t}),Q_t^a(\cdot\mid Y_{<t})\right),\\
 F(a)&=H^{-1}\sum_t\E_{X\sim P}
   \TV\!\left(P_t(\cdot\mid X_{<t}),Q_t^a(\cdot\mid X_{<t})\right).
 \end{aligned}
 \label{eq:paired-history-objects}
\end{equation}
Here \(R=\Rsmc\) by Equation~\eqref{eq:cross-history-tv}; its TV and
token estimators differ at finite \(K\). We write \(\delta_t^a\) and
\(f_t^a\) for the respective TV integrands. The three paths are paired
within an action. Across actions we pair documents, whose realized
reference continuations may differ. Documents are the statistical unit,
with equal weighting of the six strata.

Each study uses 36 development documents (six per stratum), disjoint from
both evaluation cohorts. Development variance and measured cost determine
evaluation sample sizes before evaluation outcomes are observed, using
$K=32$ development replicates in the first study and $K=16$ at higher retention.

We bootstrap documents within strata, pairing sampled documents across
actions: 10,000 draws for calibration and higher-retention analyses, and
2,000 for first-study exploration. First-study occurrence, exposure, and
time profiles are post-outcome exploratory analyses with pointwise 95\%
intervals. The independent 90\% study prespecifies eight pooled branch-TV
and saturation endpoints with family-adjusted intervals; supporting
90\% analyses are descriptive with pointwise 95\% intervals.

\subsection{Entry into divergence organizes cumulative differences}

SnapKV-half delays first divergence.
On Llama, the fraction diverging by step 32 is 61.05\%, compared with
92.35\% for SnapKV-512 and 90.81\% for recent-half; by step 128 these
fractions are 92.29\%, 98.14\%, and 97.66\%. Qwen shows the same ordering
(Figure~\ref{fig:entry-exposure}). Among trajectories diverging within the
horizon, mean entry times are 30.37 versus 11.69 and 10.59 steps on
Llama, and 29.37 versus 12.18 and 12.14 on Qwen.

\begin{figure}[t]
\centering
\actionlegend\par\smallskip
\begin{subfigure}[b]{.485\linewidth}
\centering\input{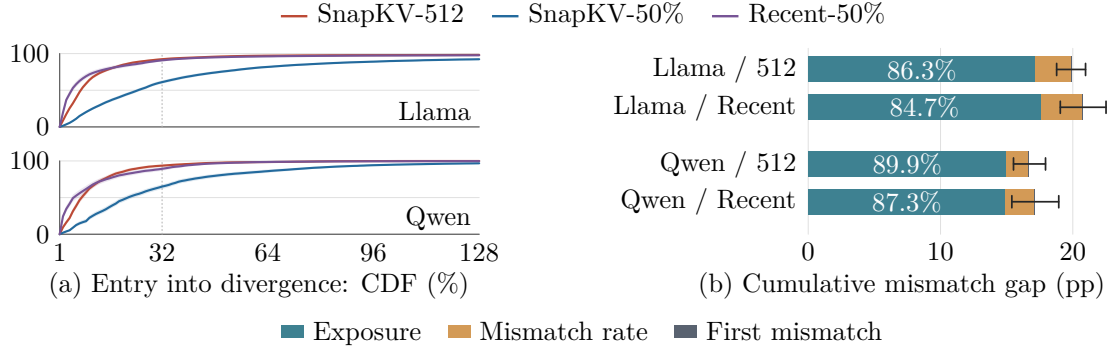}
\caption{Entry into divergence: CDF (\%)}
\end{subfigure}\hfill
\begin{subfigure}[b]{.485\linewidth}
\centering\begin{tikzpicture}[x=1cm,y=1cm,font=\figurefont,line cap=round]
\path[use as bounding box] (0,0) rectangle (6.7,3.0);
\draw[black!12,line width=.3pt] (2.15000,0.40000) -- (2.15000,2.86000);
\node[inner sep=0pt,] at (2.15000,0.17000) {0};
\draw[black!12,line width=.3pt] (3.90000,0.40000) -- (3.90000,2.86000);
\node[inner sep=0pt,] at (3.90000,0.17000) {10};
\draw[black!12,line width=.3pt] (5.65000,0.40000) -- (5.65000,2.86000);
\node[inner sep=0pt,] at (5.65000,0.17000) {20};
\node[inner sep=0pt,anchor=east] at (2.00000,2.58000) {Llama / 512};
\path[fill=figexposure] (2.15000,2.41000) rectangle (5.15550,2.75000);
\path[fill=figrate] (5.15550,2.41000) rectangle (5.62405,2.75000);
\path[fill=figfloor] (5.62405,2.41000) rectangle (5.63206,2.75000);
\node[inner sep=0pt,text=white] at (3.65275,2.58000) {86.3\%};
\draw[black!85,line width=.65pt] (5.43636,2.58000) -- (5.81871,2.58000);
\draw[black!85,line width=.65pt] (5.43636,2.50000) -- (5.43636,2.66000);
\draw[black!85,line width=.65pt] (5.81871,2.50000) -- (5.81871,2.66000);
\node[inner sep=0pt,anchor=east] at (2.00000,2.08000) {Llama / Recent};
\path[fill=figexposure] (2.15000,1.91000) rectangle (5.22927,2.25000);
\path[fill=figrate] (5.22927,1.91000) rectangle (5.78029,2.25000);
\path[fill=figfloor] (5.78029,1.91000) rectangle (5.78763,2.25000);
\node[inner sep=0pt,text=white] at (3.68963,2.08000) {84.7\%};
\draw[black!85,line width=.65pt] (5.48469,2.08000) -- (6.08978,2.08000);
\draw[black!85,line width=.65pt] (5.48469,2.00000) -- (5.48469,2.16000);
\draw[black!85,line width=.65pt] (6.08978,2.00000) -- (6.08978,2.16000);
\node[inner sep=0pt,anchor=east] at (2.00000,1.33000) {Qwen / 512};
\path[fill=figexposure] (2.15000,1.16000) rectangle (4.77372,1.50000);
\path[fill=figrate] (4.77372,1.16000) rectangle (5.06535,1.50000);
\path[fill=figfloor] (5.06535,1.16000) rectangle (5.06966,1.50000);
\node[inner sep=0pt,text=white] at (3.46186,1.33000) {89.9\%};
\draw[black!85,line width=.65pt] (4.86572,1.33000) -- (5.28922,1.33000);
\draw[black!85,line width=.65pt] (4.86572,1.25000) -- (4.86572,1.41000);
\draw[black!85,line width=.65pt] (5.28922,1.25000) -- (5.28922,1.41000);
\node[inner sep=0pt,anchor=east] at (2.00000,0.83000) {Qwen / Recent};
\path[fill=figexposure] (2.15000,0.66000) rectangle (4.76184,1.00000);
\path[fill=figrate] (4.76184,0.66000) rectangle (5.13786,1.00000);
\path[fill=figfloor] (5.13786,0.66000) rectangle (5.14194,1.00000);
\node[inner sep=0pt,text=white] at (3.45592,0.83000) {87.3\%};
\draw[black!85,line width=.65pt] (4.84373,0.83000) -- (5.46425,0.83000);
\draw[black!85,line width=.65pt] (4.84373,0.75000) -- (4.84373,0.91000);
\draw[black!85,line width=.65pt] (5.46425,0.75000) -- (5.46425,0.91000);
\end{tikzpicture}
\caption{Cumulative mismatch gap (pp)}
\end{subfigure}\par\smallskip
\figblock{figexposure}\ Exposure\quad\figblock{figrate}\ Mismatch rate\quad
\figblock{figfloor}\ First mismatch
\caption{Post-divergence exposure accounts for most of the mismatch gap.
Left: \(\Pcal(\tau\le t)\) versus generation position \(t\), including
trajectories agreeing through \(H=128\).
Right: contrasts with SnapKV-half follow Equation~\eqref{eq:exposure-contrast};
component lengths and exposure shares are point estimates. CDF shading and
total error bars give pointwise 95\% document-bootstrap intervals. Recent-half
uses identical prompt-cache bytes. Exploratory analysis: \(N=288,K=64\) per action and model.}
\label{fig:entry-exposure}
\end{figure}

SnapKV-512 minus SnapKV-half gives 19.90 and 16.68 percentage points
(pp) of token mismatch on Llama and Qwen. The exposure terms contribute
17.17 [16.20, 18.13] and 14.99 [13.87, 16.22] pp, respectively: 86.3\%
and 89.9\% of the point-estimated gaps. The equal-budget recent-half
contrasts give 20.79 and 17.10 pp, with exposure shares 84.7\% and
87.3\%. All 24 model--stratum--contrast point decompositions have an
exposure majority (78.1--96.1\%). Equation~\eqref{eq:exposure-contrast}
holds exactly for the empirical token-mismatch moments; each component
share divides its contribution by the observed mismatch gap. Thus the
pattern is also present when retention quantity is fixed and the selected
entries change.

Mismatch during post-divergence time is high for every action:
\(C_H\) ranges from 89.74\% to 93.68\% on Llama and 93.59\% to
96.20\% on Qwen. Its between-action differences contribute
1.67--3.15 pp to the aggregate gaps.

\subsection{Generation time and divergence time reveal different dynamics}
\label{sec:two-clocks}

Figure~\ref{fig:two-clocks} separates two temporal patterns: intervention
gaps narrow along generation positions, while discrepancy stays high
within branches aligned at divergence. For the latter, each action's fixed
\(\tau\le64\) cohort gives every member 64 strictly post-divergence steps
and ensures \(\tau+j\le H\) for \(1\le j\le64\).

For SnapKV-half, branch TV rises from 84.45\% to 91.56\% on Llama and
from 85.28\% to 95.76\% on Qwen between the early and late windows
(Figure~\ref{fig:two-clocks}b,d).

\begin{figure}[t]
\centering
\actionlegend\par\smallskip
\begin{subfigure}[b]{.485\linewidth}
\centering\begin{tikzpicture}[x=1cm,y=1cm,font=\figurefont,line cap=round]
\path[use as bounding box] (0,0) rectangle (6.7,2.0);
\draw[black!12,line width=.3pt] (0.85000,0.48000) -- (6.40000,0.48000);
\node[inner sep=0pt,anchor=east] at (0.73000,0.48000) {0};
\draw[black!12,line width=.3pt] (0.85000,1.14000) -- (6.40000,1.14000);
\node[inner sep=0pt,anchor=east] at (0.73000,1.14000) {20};
\draw[black!12,line width=.3pt] (0.85000,1.80000) -- (6.40000,1.80000);
\node[inner sep=0pt,anchor=east] at (0.73000,1.80000) {40};
\draw[black!75,line width=.45pt] (0.85000,1.80000) -- (0.85000,0.48000) -- (6.40000,0.48000);
\draw[black!75,line width=.45pt] (0.85000,0.48000) -- (0.85000,0.41000);
\node[inner sep=0pt,anchor=north] at (0.85000,0.35000) {1--32};
\draw[black!75,line width=.45pt] (3.62500,0.48000) -- (3.62500,0.41000);
\node[inner sep=0pt,anchor=north] at (3.62500,0.35000) {33--64};
\draw[black!75,line width=.45pt] (6.40000,0.48000) -- (6.40000,0.41000);
\node[inner sep=0pt,anchor=north] at (6.40000,0.35000) {65--128};
\draw[figred,line width=.85pt] (0.85000,1.56113) -- (3.62500,1.27390) -- (6.40000,0.85524);
\fill[figred] (0.85000,1.56113) circle[radius=.045cm];
\fill[figred] (3.62500,1.27390) circle[radius=.045cm];
\fill[figred] (6.40000,0.85524) circle[radius=.045cm];
\draw[figpurple,line width=.85pt] (0.85000,1.70740) -- (3.62500,1.25071) -- (6.40000,0.85295);
\fill[figpurple] (0.85000,1.70740) circle[radius=.045cm];
\fill[figpurple] (3.62500,1.25071) circle[radius=.045cm];
\fill[figpurple] (6.40000,0.85295) circle[radius=.045cm];
\end{tikzpicture}
\caption{Llama: generation-clock gap (pp)}
\end{subfigure}\hfill
\begin{subfigure}[b]{.485\linewidth}
\centering\begin{tikzpicture}[x=1cm,y=1cm,font=\figurefont,line cap=round]
\path[use as bounding box] (0,0) rectangle (6.7,2.0);
\draw[black!12,line width=.3pt] (0.85000,0.60000) -- (6.40000,0.60000);
\node[inner sep=0pt,anchor=east] at (0.73000,0.60000) {80};
\draw[black!12,line width=.3pt] (0.85000,1.20000) -- (6.40000,1.20000);
\node[inner sep=0pt,anchor=east] at (0.73000,1.20000) {90};
\draw[black!12,line width=.3pt] (0.85000,1.80000) -- (6.40000,1.80000);
\node[inner sep=0pt,anchor=east] at (0.73000,1.80000) {100};
\draw[black!75,line width=.45pt] (0.85000,1.80000) -- (0.85000,0.48000) -- (6.40000,0.48000);
\draw[black!75,line width=.45pt] (0.85000,0.48000) -- (0.85000,0.41000);
\node[inner sep=0pt,anchor=north] at (0.85000,0.35000) {1};
\draw[black!75,line width=.45pt] (2.17143,0.48000) -- (2.17143,0.41000);
\node[inner sep=0pt,anchor=north] at (2.17143,0.35000) {16};
\draw[black!75,line width=.45pt] (3.58095,0.48000) -- (3.58095,0.41000);
\node[inner sep=0pt,anchor=north] at (3.58095,0.35000) {32};
\draw[black!75,line width=.45pt] (4.99048,0.48000) -- (4.99048,0.41000);
\node[inner sep=0pt,anchor=north] at (4.99048,0.35000) {48};
\draw[black!75,line width=.45pt] (6.40000,0.48000) -- (6.40000,0.41000);
\node[inner sep=0pt,anchor=north] at (6.40000,0.35000) {64};
\path[fill=black,fill opacity=.04] (0.85000,0.48000) rectangle (2.17143,1.80000);
\path[fill=black,fill opacity=.04] (5.07857,0.48000) rectangle (6.40000,1.80000);
\draw[figred,line width=.85pt] (0.85000,0.87726) -- (0.93810,0.79792) -- (1.02619,0.83320) -- (1.11429,0.98525) -- (1.20238,0.97120) -- (1.29048,1.01878) -- (1.37857,1.03797) -- (1.46667,1.14802) -- (1.55476,1.14698) -- (1.64286,1.12350) -- (1.73095,1.15552) -- (1.81905,1.24309) -- (1.90714,1.24842) -- (1.99524,1.24259) -- (2.08333,1.25490) -- (2.17143,1.30309) -- (2.25952,1.30339) -- (2.34762,1.31079) -- (2.43571,1.32243) -- (2.52381,1.36327) -- (2.61190,1.37195) -- (2.70000,1.36998) -- (2.78810,1.37973) -- (2.87619,1.39690) -- (2.96429,1.39146) -- (3.05238,1.39953) -- (3.14048,1.40026) -- (3.22857,1.42276) -- (3.31667,1.41851) -- (3.40476,1.42551) -- (3.49286,1.42469) -- (3.58095,1.44856) -- (3.66905,1.43926) -- (3.75714,1.45117) -- (3.84524,1.44201) -- (3.93333,1.45636) -- (4.02143,1.46572) -- (4.10952,1.47719) -- (4.19762,1.47979) -- (4.28571,1.47244) -- (4.37381,1.48070) -- (4.46190,1.48330) -- (4.55000,1.47444) -- (4.63810,1.47388) -- (4.72619,1.47710) -- (4.81429,1.46987) -- (4.90238,1.47252) -- (4.99048,1.47400) -- (5.07857,1.47957) -- (5.16667,1.49058) -- (5.25476,1.48141) -- (5.34286,1.48469) -- (5.43095,1.48049) -- (5.51905,1.47662) -- (5.60714,1.47617) -- (5.69524,1.48106) -- (5.78333,1.48047) -- (5.87143,1.47992) -- (5.95952,1.47838) -- (6.04762,1.48274) -- (6.13571,1.47644) -- (6.22381,1.47257) -- (6.31190,1.46858) -- (6.40000,1.46827);
\draw[figblue,line width=.85pt] (0.85000,0.77464) -- (0.93810,0.75820) -- (1.02619,0.76278) -- (1.11429,0.79426) -- (1.20238,0.77021) -- (1.29048,0.81144) -- (1.37857,0.81444) -- (1.46667,0.87044) -- (1.55476,0.89354) -- (1.64286,0.87570) -- (1.73095,0.88811) -- (1.81905,0.96146) -- (1.90714,0.95640) -- (1.99524,0.96106) -- (2.08333,0.98449) -- (2.17143,0.99665) -- (2.25952,1.00801) -- (2.34762,1.02331) -- (2.43571,1.03622) -- (2.52381,1.06809) -- (2.61190,1.09648) -- (2.70000,1.07384) -- (2.78810,1.09619) -- (2.87619,1.12087) -- (2.96429,1.12927) -- (3.05238,1.13087) -- (3.14048,1.14648) -- (3.22857,1.17133) -- (3.31667,1.18641) -- (3.40476,1.18541) -- (3.49286,1.18618) -- (3.58095,1.20016) -- (3.66905,1.21523) -- (3.75714,1.22325) -- (3.84524,1.21975) -- (3.93333,1.24303) -- (4.02143,1.24527) -- (4.10952,1.23749) -- (4.19762,1.24829) -- (4.28571,1.25727) -- (4.37381,1.26943) -- (4.46190,1.27910) -- (4.55000,1.27212) -- (4.63810,1.27322) -- (4.72619,1.27853) -- (4.81429,1.27501) -- (4.90238,1.27415) -- (4.99048,1.29230) -- (5.07857,1.28426) -- (5.16667,1.29067) -- (5.25476,1.28947) -- (5.34286,1.28774) -- (5.43095,1.29221) -- (5.51905,1.30090) -- (5.60714,1.29655) -- (5.69524,1.29285) -- (5.78333,1.29433) -- (5.87143,1.28959) -- (5.95952,1.28740) -- (6.04762,1.30035) -- (6.13571,1.29843) -- (6.22381,1.30419) -- (6.31190,1.29854) -- (6.40000,1.29260);
\draw[figpurple,line width=.85pt] (0.85000,1.15990) -- (0.93810,1.11490) -- (1.02619,0.81814) -- (1.11429,0.95574) -- (1.20238,1.01054) -- (1.29048,1.04522) -- (1.37857,1.11042) -- (1.46667,1.14709) -- (1.55476,1.18028) -- (1.64286,1.16358) -- (1.73095,1.18265) -- (1.81905,1.26382) -- (1.90714,1.26353) -- (1.99524,1.29738) -- (2.08333,1.29916) -- (2.17143,1.33511) -- (2.25952,1.34038) -- (2.34762,1.34945) -- (2.43571,1.35636) -- (2.52381,1.36448) -- (2.61190,1.39759) -- (2.70000,1.38686) -- (2.78810,1.39479) -- (2.87619,1.41281) -- (2.96429,1.42481) -- (3.05238,1.41626) -- (3.14048,1.42606) -- (3.22857,1.43362) -- (3.31667,1.43869) -- (3.40476,1.44547) -- (3.49286,1.43745) -- (3.58095,1.45615) -- (3.66905,1.46270) -- (3.75714,1.45556) -- (3.84524,1.45726) -- (3.93333,1.47375) -- (4.02143,1.46950) -- (4.10952,1.47882) -- (4.19762,1.48834) -- (4.28571,1.48728) -- (4.37381,1.49392) -- (4.46190,1.49224) -- (4.55000,1.48989) -- (4.63810,1.49512) -- (4.72619,1.49722) -- (4.81429,1.49333) -- (4.90238,1.49163) -- (4.99048,1.48930) -- (5.07857,1.49374) -- (5.16667,1.49302) -- (5.25476,1.49817) -- (5.34286,1.49939) -- (5.43095,1.49654) -- (5.51905,1.49907) -- (5.60714,1.49691) -- (5.69524,1.49797) -- (5.78333,1.50389) -- (5.87143,1.50298) -- (5.95952,1.49563) -- (6.04762,1.49424) -- (6.13571,1.50061) -- (6.22381,1.49170) -- (6.31190,1.49827) -- (6.40000,1.49188);
\end{tikzpicture}
\caption{Llama: post-divergence TV (\%)}
\end{subfigure}\par\smallskip
\begin{subfigure}[b]{.485\linewidth}
\centering\begin{tikzpicture}[x=1cm,y=1cm,font=\figurefont,line cap=round]
\path[use as bounding box] (0,0) rectangle (6.7,2.0);
\draw[black!12,line width=.3pt] (0.85000,0.48000) -- (6.40000,0.48000);
\node[inner sep=0pt,anchor=east] at (0.73000,0.48000) {0};
\draw[black!12,line width=.3pt] (0.85000,1.14000) -- (6.40000,1.14000);
\node[inner sep=0pt,anchor=east] at (0.73000,1.14000) {20};
\draw[black!12,line width=.3pt] (0.85000,1.80000) -- (6.40000,1.80000);
\node[inner sep=0pt,anchor=east] at (0.73000,1.80000) {40};
\draw[black!75,line width=.45pt] (0.85000,1.80000) -- (0.85000,0.48000) -- (6.40000,0.48000);
\draw[black!75,line width=.45pt] (0.85000,0.48000) -- (0.85000,0.41000);
\node[inner sep=0pt,anchor=north] at (0.85000,0.35000) {1--32};
\draw[black!75,line width=.45pt] (3.62500,0.48000) -- (3.62500,0.41000);
\node[inner sep=0pt,anchor=north] at (3.62500,0.35000) {33--64};
\draw[black!75,line width=.45pt] (6.40000,0.48000) -- (6.40000,0.41000);
\node[inner sep=0pt,anchor=north] at (6.40000,0.35000) {65--128};
\draw[figred,line width=.85pt] (0.85000,1.48850) -- (3.62500,1.14113) -- (6.40000,0.74611);
\fill[figred] (0.85000,1.48850) circle[radius=.045cm];
\fill[figred] (3.62500,1.14113) circle[radius=.045cm];
\fill[figred] (6.40000,0.74611) circle[radius=.045cm];
\draw[figpurple,line width=.85pt] (0.85000,1.57028) -- (3.62500,1.11139) -- (6.40000,0.74700);
\fill[figpurple] (0.85000,1.57028) circle[radius=.045cm];
\fill[figpurple] (3.62500,1.11139) circle[radius=.045cm];
\fill[figpurple] (6.40000,0.74700) circle[radius=.045cm];
\end{tikzpicture}
\caption{Qwen: generation-clock gap (pp)}
\end{subfigure}\hfill
\begin{subfigure}[b]{.485\linewidth}
\centering\begin{tikzpicture}[x=1cm,y=1cm,font=\figurefont,line cap=round]
\path[use as bounding box] (0,0) rectangle (6.7,2.0);
\draw[black!12,line width=.3pt] (0.85000,0.60000) -- (6.40000,0.60000);
\node[inner sep=0pt,anchor=east] at (0.73000,0.60000) {80};
\draw[black!12,line width=.3pt] (0.85000,1.20000) -- (6.40000,1.20000);
\node[inner sep=0pt,anchor=east] at (0.73000,1.20000) {90};
\draw[black!12,line width=.3pt] (0.85000,1.80000) -- (6.40000,1.80000);
\node[inner sep=0pt,anchor=east] at (0.73000,1.80000) {100};
\draw[black!75,line width=.45pt] (0.85000,1.80000) -- (0.85000,0.48000) -- (6.40000,0.48000);
\draw[black!75,line width=.45pt] (0.85000,0.48000) -- (0.85000,0.41000);
\node[inner sep=0pt,anchor=north] at (0.85000,0.35000) {1};
\draw[black!75,line width=.45pt] (2.17143,0.48000) -- (2.17143,0.41000);
\node[inner sep=0pt,anchor=north] at (2.17143,0.35000) {16};
\draw[black!75,line width=.45pt] (3.58095,0.48000) -- (3.58095,0.41000);
\node[inner sep=0pt,anchor=north] at (3.58095,0.35000) {32};
\draw[black!75,line width=.45pt] (4.99048,0.48000) -- (4.99048,0.41000);
\node[inner sep=0pt,anchor=north] at (4.99048,0.35000) {48};
\draw[black!75,line width=.45pt] (6.40000,0.48000) -- (6.40000,0.41000);
\node[inner sep=0pt,anchor=north] at (6.40000,0.35000) {64};
\path[fill=black,fill opacity=.04] (0.85000,0.48000) rectangle (2.17143,1.80000);
\path[fill=black,fill opacity=.04] (5.07857,0.48000) rectangle (6.40000,1.80000);
\draw[figred,line width=.85pt] (0.85000,1.00860) -- (0.93810,0.81524) -- (1.02619,0.64992) -- (1.11429,0.65134) -- (1.20238,0.72807) -- (1.29048,0.86148) -- (1.37857,0.89108) -- (1.46667,1.00755) -- (1.55476,1.03029) -- (1.64286,1.07720) -- (1.73095,1.14476) -- (1.81905,1.22915) -- (1.90714,1.22202) -- (1.99524,1.21232) -- (2.08333,1.22162) -- (2.17143,1.31148) -- (2.25952,1.30023) -- (2.34762,1.35493) -- (2.43571,1.37069) -- (2.52381,1.42057) -- (2.61190,1.43106) -- (2.70000,1.41384) -- (2.78810,1.41330) -- (2.87619,1.42829) -- (2.96429,1.40155) -- (3.05238,1.39741) -- (3.14048,1.39594) -- (3.22857,1.44692) -- (3.31667,1.46911) -- (3.40476,1.50459) -- (3.49286,1.51834) -- (3.58095,1.54239) -- (3.66905,1.53151) -- (3.75714,1.55112) -- (3.84524,1.53246) -- (3.93333,1.53664) -- (4.02143,1.55952) -- (4.10952,1.55551) -- (4.19762,1.56175) -- (4.28571,1.57903) -- (4.37381,1.57921) -- (4.46190,1.57699) -- (4.55000,1.58724) -- (4.63810,1.58570) -- (4.72619,1.59723) -- (4.81429,1.59730) -- (4.90238,1.60851) -- (4.99048,1.61015) -- (5.07857,1.61202) -- (5.16667,1.62310) -- (5.25476,1.61896) -- (5.34286,1.62653) -- (5.43095,1.63371) -- (5.51905,1.62971) -- (5.60714,1.63614) -- (5.69524,1.64160) -- (5.78333,1.64646) -- (5.87143,1.64616) -- (5.95952,1.65181) -- (6.04762,1.65819) -- (6.13571,1.65803) -- (6.22381,1.65849) -- (6.31190,1.65874) -- (6.40000,1.66517);
\draw[figblue,line width=.85pt] (0.85000,0.82571) -- (0.93810,0.73211) -- (1.02619,0.73186) -- (1.11429,0.74716) -- (1.20238,0.80384) -- (1.29048,0.84921) -- (1.37857,0.83139) -- (1.46667,0.89442) -- (1.55476,0.89432) -- (1.64286,0.94305) -- (1.73095,1.00624) -- (1.81905,1.03047) -- (1.90714,1.04157) -- (1.99524,1.07420) -- (2.08333,1.10625) -- (2.17143,1.15381) -- (2.25952,1.15815) -- (2.34762,1.18395) -- (2.43571,1.17978) -- (2.52381,1.20848) -- (2.61190,1.22522) -- (2.70000,1.23314) -- (2.78810,1.24675) -- (2.87619,1.28619) -- (2.96429,1.29295) -- (3.05238,1.31744) -- (3.14048,1.33710) -- (3.22857,1.36933) -- (3.31667,1.37427) -- (3.40476,1.37368) -- (3.49286,1.37803) -- (3.58095,1.38502) -- (3.66905,1.42276) -- (3.75714,1.43401) -- (3.84524,1.44596) -- (3.93333,1.45187) -- (4.02143,1.44798) -- (4.10952,1.46479) -- (4.19762,1.45960) -- (4.28571,1.47785) -- (4.37381,1.48011) -- (4.46190,1.48570) -- (4.55000,1.48589) -- (4.63810,1.48896) -- (4.72619,1.50164) -- (4.81429,1.49538) -- (4.90238,1.51239) -- (4.99048,1.52298) -- (5.07857,1.52700) -- (5.16667,1.52690) -- (5.25476,1.53305) -- (5.34286,1.52637) -- (5.43095,1.53796) -- (5.51905,1.54357) -- (5.60714,1.54056) -- (5.69524,1.54527) -- (5.78333,1.54313) -- (5.87143,1.54450) -- (5.95952,1.55752) -- (6.04762,1.55925) -- (6.13571,1.55557) -- (6.22381,1.55467) -- (6.31190,1.56432) -- (6.40000,1.56729);
\draw[figpurple,line width=.85pt] (0.85000,1.11322) -- (0.93810,1.16697) -- (1.02619,1.25646) -- (1.11429,1.24768) -- (1.20238,1.32266) -- (1.29048,1.22751) -- (1.37857,0.88971) -- (1.46667,0.97644) -- (1.55476,1.01548) -- (1.64286,1.08122) -- (1.73095,1.26506) -- (1.81905,1.28318) -- (1.90714,1.26912) -- (1.99524,1.28438) -- (2.08333,1.30768) -- (2.17143,1.36538) -- (2.25952,1.34205) -- (2.34762,1.30611) -- (2.43571,1.34405) -- (2.52381,1.37004) -- (2.61190,1.37123) -- (2.70000,1.43224) -- (2.78810,1.43809) -- (2.87619,1.40998) -- (2.96429,1.44433) -- (3.05238,1.45423) -- (3.14048,1.45102) -- (3.22857,1.49125) -- (3.31667,1.48234) -- (3.40476,1.48522) -- (3.49286,1.48437) -- (3.58095,1.50095) -- (3.66905,1.55253) -- (3.75714,1.55512) -- (3.84524,1.53971) -- (3.93333,1.54190) -- (4.02143,1.54602) -- (4.10952,1.53135) -- (4.19762,1.54678) -- (4.28571,1.57497) -- (4.37381,1.57253) -- (4.46190,1.58369) -- (4.55000,1.57944) -- (4.63810,1.60190) -- (4.72619,1.61564) -- (4.81429,1.61417) -- (4.90238,1.62527) -- (4.99048,1.62454) -- (5.07857,1.62125) -- (5.16667,1.63296) -- (5.25476,1.64201) -- (5.34286,1.64315) -- (5.43095,1.64818) -- (5.51905,1.65036) -- (5.60714,1.66199) -- (5.69524,1.67215) -- (5.78333,1.66354) -- (5.87143,1.67259) -- (5.95952,1.67830) -- (6.04762,1.67918) -- (6.13571,1.67855) -- (6.22381,1.68366) -- (6.31190,1.68542) -- (6.40000,1.67909);
\end{tikzpicture}
\caption{Qwen: post-divergence TV (\%)}
\end{subfigure}
\caption{Smaller gaps between interventions coexist with rising discrepancy
within divergent branches. Left: intervention-minus-SnapKV-half TV gaps
over generation-position blocks. Right: complete TV curves at strictly
post-divergence lags 1--64 within each action's fixed \(\tau\le64\) cohort.
Shading marks the early (1--16) and late (49--64) lag windows.
Curves are exploratory point estimates, \(N=288,K=64\); paired pointwise
95\% document-bootstrap intervals for the late-minus-early contrasts
are reported in the text.}
\label{fig:two-clocks}
\end{figure}
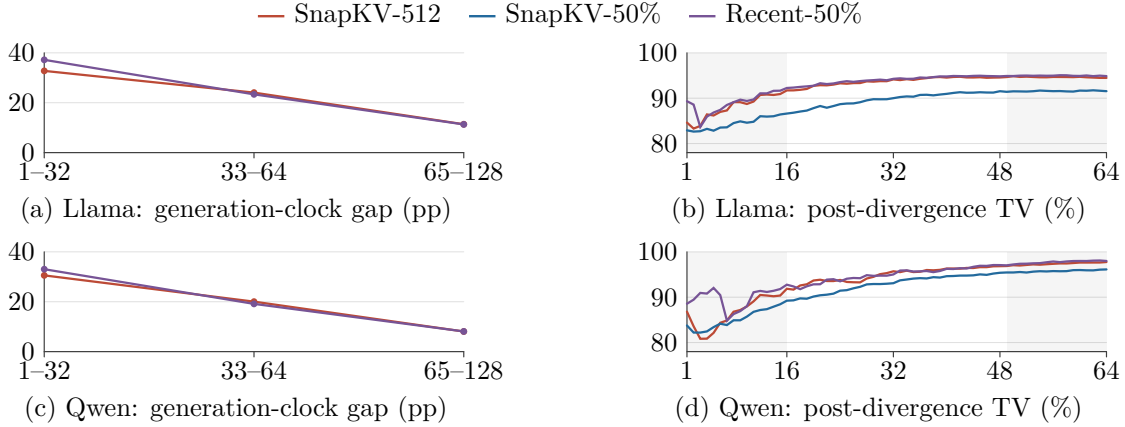

Let \(d\) denote the document. For each action, let \(I_t=\one\{X_t\ne Y_t\}\),
\(\delta_t=\delta_t(X_{<t},Y_{<t})\), and
\(\mathcal F_{t-1}=\sigma(d,X_{<t},Y_{<t})\).
Since \(\E[I_t\mid\mathcal F_{t-1}]=\delta_t\) and
\(\{\tau\ge t\}\in\mathcal F_{t-1}\),
\[
 \E[\delta_t\one\{\tau\ge t\}]
 =\E[I_t\one\{\tau\ge t\}]
 =\Pcal(\tau=t)=p_t.
\]
Partitioning by the entry time therefore gives
\begin{equation}
 \E[\delta_t]
 =p_t+\sum_{s<t}p_s\,\E[\delta_t\mid\tau=s],
 \label{eq:two-clock-mixture}
\end{equation}
with zero-mass terms defined as zero.
Generation-position averages weight branches of different ages by each
action's entry law. Later entry shortens exposure to the high-discrepancy
process, permitting narrowing aggregate gaps alongside high branch TV.

Stepwise maximality and $\{\tau=s\}\in\mathcal F_{s+j-1}$ give
\begin{equation}
 \begin{aligned}
 \E[I_t\mid\mathcal F_{t-1}]&=\delta_t,\\
 \E[I_{s+j}\one\{\tau=s\}]
 &=\E\!\left[\one\{\tau=s\}
       \E[I_{s+j}\mid\mathcal F_{s+j-1}]\right]
 =\E[\delta_{s+j}\one\{\tau=s\}],\\
 \E[I_{\tau+j}\mid\tau\le64]
 &=\frac{\sum_{s=1}^{64}\E[I_{s+j}\one\{\tau=s\}]}{\Pcal(\tau\le64)}
 =\E[\delta_{\tau+j}\mid\tau\le64].
 \end{aligned}
 \label{eq:branch-conditional-tv}
\end{equation}
The last line requires $\Pcal(\tau\le64)>0$ and $1\le j\le64$.

In the first study, the cohort fractions for SnapKV-512, SnapKV-half,
and recent-half are 97.18\%, 81.81\%, and 96.19\% on Llama, and
98.48\%, 86.07\%, and 98.57\% on Qwen. Their late-minus-early TV
contrasts are $6.53\,[5.74,7.36]$, $7.11\,[6.48,7.81]$, and
$5.84\,[5.11,6.58]$ pp on Llama, and $10.63\,[9.37,11.91]$,
$10.48\,[9.20,11.81]$, and $7.83\,[6.72,8.98]$ pp on Qwen.
All 36 stratum--model--action pointwise 95\% lower bounds are positive.

Local token agreement remains possible: at least 16 consecutive pre-EOS
agreements occur in 2.60\%/7.60\%/2.85\% of Llama branch trajectories
and 3.05\%/7.32\%/3.52\% of Qwen branch trajectories
(512/half/recent-half). Across these six conditions, a pre-EOS mismatch
after an earlier agreement occurs in 65.11--76.88\% of trajectories.
These fractions use the whole branch cohort as denominator.

\subsection{Persistent branch discrepancy at higher retention}
\label{sec:higher-retention}

At 90\% retention, SnapKV again enters divergence later and less often
than recent-token retention at the same prompt-cache budget. Exposure accounts for 88.6\% and
92.1\% of their 34.84-pp and 31.20-pp token-mismatch gaps on Llama and
Qwen, respectively; the shares range from 86.13\% to 95.36\%
across the 12 model--stratum comparisons.

All four prespecified branch-TV intervals lie above zero
(Figure~\ref{fig:retention-primary}a); strictly pre-EOS contrasts retain
this direction in every stratum. The study jointly evaluates the
branch contrast and the full-horizon saturation fraction
\begin{equation}
 \begin{aligned}
 B(a)&=\frac1{16}\sum_{j=49}^{64}\E[\delta_{\tau+j}^a\mid\tau\le64]
       -\frac1{16}\sum_{j=1}^{16}\E[\delta_{\tau+j}^a\mid\tau\le64],\\
 S_{.99}(a)&=H^{-1}\sum_{t=1}^H\Pcal(\delta_t^a>.99).
 \end{aligned}
 \label{eq:retention90-endpoints}
\end{equation}
The saturation fractions in Table~\ref{tab:retention90-saturation} place
the branch increases in a regime with substantial mass near the upper
endpoint of TV.

\begin{table}[t]
\centering\small
\setlength{\tabcolsep}{4pt}
\begin{tabular}{llrr}
\toprule
Model & Selector & $S_{.99}$, \% [99.375\% CI] & Cohort docs / paths\\
\midrule
Llama & SnapKV-90 & $28.44\;[26.68,30.24]$ & 288 / 4,723\\
 & recent-90 & $56.18\;[53.77,58.61]$ & 288 / 8,119\\
Qwen & SnapKV-90 & $43.63\;[40.94,46.35]$ & 286 / 5,275\\
 & recent-90 & $73.65\;[71.02,76.14]$ & 287 / 8,301\\
\bottomrule
\end{tabular}
\caption{Full-horizon TV saturation and branch-cohort support at higher
retention. These four fractions and the
four branch contrasts in Figure~\ref{fig:retention-primary}a form the
prespecified family of eight endpoints, with document-bootstrap intervals
at $\alpha=.05/8$. The last column counts documents and trajectories in
each action's fixed $\tau\le64$ cohort.}
\label{tab:retention90-saturation}
\end{table}

\begin{figure}[t]
\centering
\begin{subfigure}[b]{.485\linewidth}
\centering\begin{tikzpicture}[x=1cm,y=1cm,font=\figurefont,line cap=round]
\path[use as bounding box] (0,0) rectangle (6.7,2.5);
\draw[black!12,line width=.3pt] (2.45000,0.42000) -- (2.45000,2.40000);
\node[inner sep=0pt,] at (2.45000,0.16000) {0};
\draw[black!12,line width=.3pt] (4.22273,0.42000) -- (4.22273,2.40000);
\node[inner sep=0pt,] at (4.22273,0.16000) {5};
\draw[black!12,line width=.3pt] (5.99545,0.42000) -- (5.99545,2.40000);
\node[inner sep=0pt,] at (5.99545,0.16000) {10};
\node[inner sep=0pt,anchor=east] at (2.30000,2.15000) {Llama / SnapKV};
\draw[figexposure,line width=.8pt] (4.21284,2.15000) -- (5.00054,2.15000);
\draw[figexposure,line width=.8pt] (4.21284,2.09000) -- (4.21284,2.21000);
\draw[figexposure,line width=.8pt] (5.00054,2.09000) -- (5.00054,2.21000);
\fill[figexposure] (4.58964,2.15000) circle[radius=.06cm];
\node[inner sep=0pt,anchor=east] at (2.30000,1.65000) {Llama / Recent};
\draw[figexposure,line width=.8pt] (4.42329,1.65000) -- (5.22466,1.65000);
\draw[figexposure,line width=.8pt] (4.42329,1.59000) -- (4.42329,1.71000);
\draw[figexposure,line width=.8pt] (5.22466,1.59000) -- (5.22466,1.71000);
\fill[figexposure] (4.82220,1.65000) circle[radius=.06cm];
\node[inner sep=0pt,anchor=east] at (2.30000,1.15000) {Qwen / SnapKV};
\draw[figexposure,line width=.8pt] (4.90113,1.15000) -- (5.92313,1.15000);
\draw[figexposure,line width=.8pt] (4.90113,1.09000) -- (4.90113,1.21000);
\draw[figexposure,line width=.8pt] (5.92313,1.09000) -- (5.92313,1.21000);
\fill[figexposure] (5.38084,1.15000) circle[radius=.06cm];
\node[inner sep=0pt,anchor=east] at (2.30000,0.65000) {Qwen / Recent};
\draw[figexposure,line width=.8pt] (4.61950,0.65000) -- (5.63286,0.65000);
\draw[figexposure,line width=.8pt] (4.61950,0.59000) -- (4.61950,0.71000);
\draw[figexposure,line width=.8pt] (5.63286,0.59000) -- (5.63286,0.71000);
\fill[figexposure] (5.11031,0.65000) circle[radius=.06cm];
\end{tikzpicture}
\caption{Late minus early branch TV (pp)}
\end{subfigure}\hfill
\begin{subfigure}[b]{.485\linewidth}
\centering\begin{tikzpicture}[x=1cm,y=1cm,font=\figurefont,line cap=round]
\path[use as bounding box] (0,0) rectangle (6.7,2.5);
\draw[black!12,line width=.3pt] (1.25000,0.42000) -- (1.25000,1.95000);
\node[inner sep=0pt,] at (1.25000,0.16000) {0};
\draw[black!12,line width=.3pt] (2.66096,0.42000) -- (2.66096,1.95000);
\node[inner sep=0pt,] at (2.66096,0.16000) {10};
\draw[black!12,line width=.3pt] (4.07192,0.42000) -- (4.07192,1.95000);
\node[inner sep=0pt,] at (4.07192,0.16000) {20};
\draw[black!12,line width=.3pt] (5.48288,0.42000) -- (5.48288,1.95000);
\node[inner sep=0pt,] at (5.48288,0.16000) {30};
\node[inner sep=0pt,anchor=east] at (1.10000,1.65000) {Llama};
\path[fill=figexposure] (1.25000,1.43000) rectangle (5.01982,1.87000);
\path[fill=figrate] (5.01982,1.43000) rectangle (5.28005,1.87000);
\node[inner sep=0pt,text=white] at (3.13491,1.65000) {26.72};
\node[inner sep=0pt,anchor=west] at (5.40005,1.65000) {6.46\%};
\node[inner sep=0pt,anchor=east] at (1.10000,0.85000) {Qwen};
\path[fill=figexposure] (1.25000,0.63000) rectangle (4.29465,1.07000);
\path[fill=figrate] (4.29465,0.63000) rectangle (4.30149,1.07000);
\node[inner sep=0pt,text=white] at (2.77232,0.85000) {21.58};
\node[inner sep=0pt,anchor=west] at (4.42149,0.85000) {0.22\%};
\path[fill=figexposure] (1.20000,2.23000) rectangle (1.50000,2.38000);
\node[inner sep=0pt,anchor=west] at (1.60000,2.30000) {Pre-EOS};
\path[fill=figrate] (3.55000,2.23000) rectangle (3.85000,2.38000);
\node[inner sep=0pt,anchor=west] at (3.95000,2.30000) {EOS-related};
\end{tikzpicture}
\caption{Selector difference in $R-F$ (pp)}
\end{subfigure}
\caption{Higher retention preserves branch discrepancy; most of the selector
history-gap difference accrues before EOS. Independent 90\% study:
\(N=288,K=32\). (a) Lags 49--64 minus 1--16 in each action's
\(\tau\le64\) cohort; prespecified 99.375\% document-bootstrap intervals
for the eight-endpoint family.
(b) Recent minus SnapKV, with additive pre-EOS and EOS-related components;
end labels give EOS shares of the total. Bars and shares are descriptive
point estimates.}
\label{fig:retention-primary}
\end{figure}
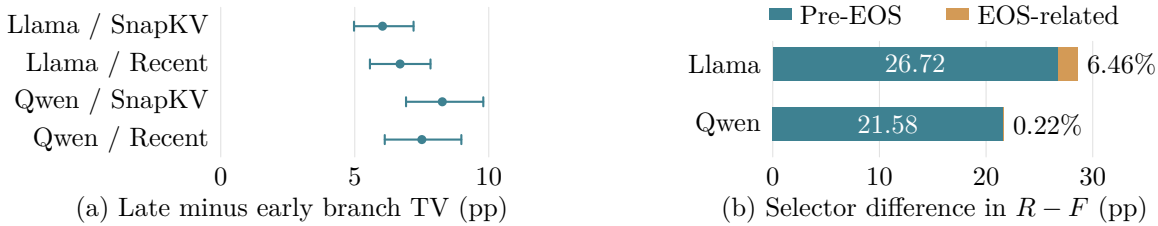

The branch cohorts contain 51.25\%/88.10\% of Llama trajectories and
57.24\%/90.07\% of Qwen trajectories (SnapKV/recent). Full-horizon
divergence probabilities are 74.24\%/94.88\% and 80.88\%/97.75\%,
with conditional mean entry times of 49.61/20.19 and 47.74/21.47 steps.
Early-to-late branch TV is 82.64\% to 88.68\% and 87.44\% to 94.13\%
on Llama, and 85.87\% to 94.14\% and 89.56\% to 97.06\% on Qwen.
For comparison, first-study saturation fractions are
58.15\%/44.92\%/62.41\% on Llama and 81.23\%/65.53\%/82.48\% on Qwen
(512/half/recent-half).

\subsection{EOS and the location of discrepancy}
\label{sec:eos}

Rollout and fixed-history TV agree through the first divergent token;
\(R-F\) accumulates later. For each action, let $e_X,e_Y$ be the first
EOS output positions (infinite if absent), $e=\min(e_X,e_Y)$, and
$e'=\max(e_X,e_Y)$. Within $\{1,\ldots,H\}$, define
\[
 S_0=\{t:t<e\},\quad S_1=\{t:t=e\le H\},\quad
 S_2=\{t:e<t\le e'\},\quad S_3=\{t:e'<t\}.
\]
These sets separate pre-EOS, the first EOS output, one side absorbed,
and both sides absorbed. With fixed-history TV $f_t^a$ and
$\Delta$ denoting recent-90 minus SnapKV-90,
\begin{equation}
 G_r(a)=\frac1H\E\!\left[\sum_{t\in S_r}
       (\delta_t^a-f_t^a)\right],\qquad
 \Delta(R-F)=\sum_{r=0}^3\Delta G_r,\qquad G_3(a)=0.
 \label{eq:trajectory-eos-partition}
\end{equation}
The pre-EOS component accounts for most of the selector difference
(Figure~\ref{fig:retention-primary}b). The EOS-related component is
$\sum_{r=1}^3\Delta G_r$; its ratio to $\Delta(R-F)$ gives the
descriptive EOS share.

The total selector differences are $28.56\,[26.66,30.48]$ pp on Llama
and $21.63\,[19.58,23.66]$ pp on Qwen (pointwise 95\% intervals); their
pre-EOS components have intervals $[25.11,28.31]$ and $[19.52,23.61]$
pp. All 12 model--stratum gaps and pre-EOS components have positive
pointwise lower bounds. Llama's EOS shares range from 0.86\% to 3.68\%
on HELMET and 9.52\% to 11.65\% on RULER; Qwen's range is approximately
$-0.024\%$ to 1.25\%. A signed selector difference can have a negative
component.

For the first study, pre-EOS components of 512-minus-half
$\Delta(R-F)$ are 8.83 of 10.39 pp on Llama and 4.08 of 4.06 pp on
Qwen; recent-half minus half gives 10.65 of 12.14 pp and 4.70 of
4.70 pp. The Qwen HELMET--32K 512-minus-half estimate is
$-0.35\,[-2.84,2.05]$ pp, with its interval spanning both signs.

For a lag block $J$, the strict pre-EOS branch mean is
\[
 b_J^{\rm pre}=
 \frac{\E[\sum_{j\in J}\one\{\tau\le64,\ \tau+j<e\}\delta_{\tau+j}]}
 {\E[\sum_{j\in J}\one\{\tau\le64,\ \tau+j<e\}]},
\]
with summands zero when $\tau>64$.
At 90\% retention, $b_{49:64}^{\rm pre}-b_{1:16}^{\rm pre}$ is
$5.94\,[5.14,6.77]$ and $6.83\,[6.01,7.65]$ pp on Llama, and
$8.26\,[7.27,9.32]$ and $7.51\,[6.51,8.57]$ pp on Qwen
(SnapKV/recent; pointwise 95\%). Early/late pre-EOS step fractions
are 97.95\%/94.87\% and 98.63\%/94.54\% on Llama; RULER late fractions
are approximately 90\%--92\%. Qwen's fractions remain above
approximately 99.8\%. Each mean averages surviving steps in its own block,
whose contributing trajectory sets can differ.

\subsection{Task-dependent timing and persistent branch discrepancy}
\label{sec:task-strata}

For generation-position block $J$, write
$g_J=|J|^{-1}\sum_{t\in J}\E[\delta_t^{\rm recent90}-\delta_t^{\rm SnapKV90}]$.
The pooled values for $J=1{:}32,33{:}64,65{:}128$ are
37.89, 41.35, and 30.11 pp on Llama and 36.09, 37.63, and 25.53 pp on
Qwen. Table~\ref{tab:trajectory-stratum-branches} separates the two
clocks by stratum: $g_{65:128}-g_{1:32}$ is negative in every HELMET
model--length comparison and positive in every RULER comparison.

\begin{table}[t]
\centering\small
\setlength{\tabcolsep}{4pt}
\begin{tabular}{llrrr}
\toprule
Model & Stratum & $\Delta g$ & $B$: SnapKV [95\% CI] & $B$: recent [95\% CI]\\
\midrule
Llama & HELMET--4K & $-21.72$ & $4.21\;[2.51,6.10]$ & $6.18\;[4.73,7.74]$\\
 & HELMET--16K & $-18.14$ & $7.43\;[5.27,9.87]$ & $5.01\;[3.90,6.16]$\\
 & HELMET--32K & $-14.34$ & $8.48\;[6.69,10.18]$ & $6.28\;[4.96,7.71]$\\
 & RULER--4K & $+3.25$ & $5.01\;[3.13,7.03]$ & $7.26\;[4.36,10.38]$\\
 & RULER--16K & $+0.50$ & $5.39\;[3.59,7.52]$ & $8.54\;[6.10,11.10]$\\
 & RULER--32K & $+3.73$ & $5.49\;[4.07,7.02]$ & $7.16\;[5.18,9.16]$\\
\midrule
Qwen & HELMET--4K & $-25.84$ & $6.60\;[4.85,8.45]$ & $6.08\;[3.86,8.87]$\\
 & HELMET--16K & $-26.36$ & $11.46\;[8.19,15.12]$ & $4.86\;[3.39,6.47]$\\
 & HELMET--32K & $-19.54$ & $9.12\;[7.40,10.88]$ & $7.23\;[4.79,10.25]$\\
 & RULER--4K & $+0.27$ & $7.25\;[4.48,10.57]$ & $8.81\;[6.63,11.28]$\\
 & RULER--16K & $+0.28$ & $7.01\;[5.10,9.03]$ & $10.00\;[6.71,13.81]$\\
 & RULER--32K & $+7.84$ & $7.81\;[5.94,9.59]$ & $8.68\;[6.60,10.92]$\\
\bottomrule
\end{tabular}
\caption{Two clocks at 90\% retention, in pp. The generation-clock
column is $\Delta g=g_{65:128}-g_{1:32}$; $B$ compares
late and early lags within each selector's fixed branch cohort.
Intervals give descriptive, pointwise 95\% uncertainty.}
\label{tab:trajectory-stratum-branches}
\label{tab:trajectory-stratum-timing}
\end{table}

All 24 branch contrasts have positive pointwise lower bounds.
Their strict pre-EOS versions range from 4.18 to 11.46 pp, also with
positive lower bounds (the smallest is 2.47 pp).
In the first study, the pooled 512-minus-half generation-clock gaps
are 32.76, 24.06, and 11.37 pp on Llama and 30.56, 20.03, and 8.06 pp
on Qwen; the RULER--4K gap rises in the middle block.

\subsection{Aggregate prediction and temporal resolution}
\label{sec:aggregate-resolution}

Calibration predicts the population-average SnapKV-512 minus
SnapKV-half gap in $R$ from either fixed-history TV or the first 32
generation positions. For each model and observation, six-fold selection
on development documents
minimizes document-level paired-contrast mean
squared error over a stratum/action mean, an uncalibrated flat predictor,
and ridge regression ($\alpha=10$). Ridge uses stratum/action indicators
and either full-horizon $F$ or the replicate-averaged short-window features
\[
 \frac1K\sum_{k=1}^K
 \left(\frac1H\sum_{t=1}^{32}I_t^{(k)},\;
 \one\{\tau^{(k)}\le32\},\;
 \one\{e_X^{(k)}\le32\},\;
 \one\{e_Y^{(k)}\le32\}\right),
\]
where $e_X,e_Y$ are the first EOS positions. Inputs are centered and scaled
on training documents, with an unpenalized intercept; each action prediction
is clipped to $[0,1]$ before taking its contrast.

Using generation positions 1--32, development-selected ridge predictors
recover the population-average SnapKV-512 minus SnapKV-half gap within
the prespecified \(\pm2\)-pp margin in both models
(Table~\ref{tab:trajectory-calibration}). For error
$e=\Delta R-\widehat{\Delta R}$, equivalence requires
$\mathrm{CI}_{1-0.05/4}(e)\subset[-2,2]\ \mathrm{pp}$.
With fixed-history observations, Qwen's selected cell-mean predictor
satisfies this criterion using stratum/action identity alone; Llama's
ridge interval spans values inside and above the equivalence margin.

\begin{table}[t]
\centering\small
\setlength{\tabcolsep}{4pt}
\begin{tabular}{llrl}
\toprule
Model and observation & Predictor & Error, pp [98.75\% CI] & Decision\\
\midrule
Llama, fixed history & Ridge & $1.127\;[-0.304,2.570]$ & Unresolved\\
Llama, positions 1--32 & Ridge & $-0.388\;[-1.062,0.301]$ & Equivalent\\
Qwen, fixed history & Cell mean & $0.370\;[-1.162,1.965]$ & Equivalent\\
Qwen, positions 1--32 & Ridge & $0.114\;[-0.867,1.135]$ & Equivalent\\
\bottomrule
\end{tabular}
\caption{Development-selected aggregate predictions on independent
documents. Equivalence concerns the six-stratum population mean and
the stated action pair, with a prespecified $\pm2$-pp margin.}
\label{tab:trajectory-calibration}
\end{table}

The short-window predictors have document-contrast MAEs of 4.09 pp on
Llama and 4.92 pp on Qwen. Their inputs use generation positions 1--32;
the theoretical depth \(L\) begins at divergence.
Short-window document-contrast
RMSE against the finite-$K$ targets is 5.16 and 7.09 pp on Llama and
Qwen, respectively. Fixed-history predictors give MAE/RMSE of
8.73/11.04 and 10.03/13.18 pp. These errors summarize document-level paired contrasts.

Table~\ref{tab:trajectory-uncalibrated} compares the corresponding
uncalibrated summaries.

\begin{table}[t]
\centering\small
\begin{tabular}{lrr}
\toprule
Uncalibrated quantity, pp & Llama & Qwen\\
\midrule
Full-$H$ rollout-TV gap & 19.89 & 16.68\\
Fixed-history gap & 9.50 & 12.63\\
Positions 1--32 flat-extrapolation gap & 32.78 & 30.57\\
Full minus fixed history & $10.39\;[9.26,11.55]$ & $4.06\;[2.60,5.46]$\\
Full minus flat extrapolation & $-12.89\;[-13.97,-11.83]$ & $-13.89\;[-15.39,-12.40]$\\
\bottomrule
\end{tabular}
\caption{Prespecified supporting contrasts for SnapKV-512 minus
SnapKV-half, with pointwise 95\% intervals. Flat extrapolation uses
mean token mismatch at positions 1--32.}
\label{tab:trajectory-uncalibrated}
\end{table}

\section{Discussion and Conclusion}
\label{sec:conclusion}

The trajectory decomposition gives a temporal interpretation of the
benefits of KV retention. In the first study, post-divergence exposure
accounts for 85--90\% of four aggregate mismatch gaps, including
equal-budget selector comparisons. Lower cumulative mismatch is associated
chiefly with later and less frequent entry into divergence, while
discrepancy remains high within divergent branches. The independent 90\%
study strengthens this account: all four prespecified late-minus-early
branch-TV contrasts are positive, ranging from 6.03 to 8.27 pp.
EOS decompositions locate most of the selector history-gap difference
before termination.

The two time axes explain how these observations coexist.
Equation~\eqref{eq:two-clock-mixture} expresses generation-position TV as
a mixture over entry times and branch ages. At 90\% retention,
generation-clock gaps narrow on HELMET and widen on RULER, while
late-window branch TV increases across all studied strata.
Entry-aligned measurements thus reveal a shared pattern of persistent
discrepancy beneath task-dependent aggregate dynamics.

The theory makes this temporal resolution precise. A finite
divergence-aligned profile determines an observed contribution \(R_L\)
and remaining capacity \(B_L\); the constructive sharpness result
characterizes every compatible cumulative risk over unrestricted kernel
pairs. Residual-branch conditional measurement jointly estimates
occurrence, occupation, and the realized tail with an exact variance
comparison. These results connect what a trajectory measurement observes
to the cumulative quantity it aims to describe.

Aggregate prediction provides a complementary use of early observations.
Predictors calibrated on development documents and using generation
positions 1--32 recover the specified
population-average SnapKV-512 minus SnapKV-half contrast within the
prespecified \(\pm2\)-pp margin in both models. Document-level contrast
MAEs are 4.09 and 4.92 pp, distinguishing the resolution of population
averages from that of individual documents. In the studied setting,
calibrated early observations summarize the mean effect, while complete
trajectories reveal its composition in divergence timing, exposure,
and subsequent mismatch. Reporting these components gives cumulative
comparisons a direct temporal interpretation.

\section*{AI use}
AI assistants supported code development, literature search, and
checks of mathematical arguments. The authors are responsible for the
claims, citations, analyses, and final presentation.

\clearpage
\begingroup
\fontsize{10}{11.5}\selectfont
\setlength{\bibsep}{2pt plus 0.5pt}
\raggedright
\bibliography{references,iclr2027-additional-references}
\bibliographystyle{plainnat}
\endgroup

\end{document}